%% file: rrtsharp_pi_report_final.tex
\documentclass[letterpaper, 11pt]{article}

\usepackage{remark}
\newremark{remark}{Remark}

\input{packages.tex}
\input{commands.tex}

\newcommand{\QED}{\hfill$\Box$}

\usepackage{enumerate}
\newtheorem{corollary}{Corollary}

\title{\LARGE \bf
Incremental Sampling-based Motion Planners \\ Using Policy Iteration Methods
}

\author{Oktay Arslan\thanks{Oktay Arslan performed this research during his doctoral studies as a Robotics, PhD student while at
Georgia Institute of Technology, Atlanta, GA 30332-0150, USA.
}\\
\textit{Mobility and Robotic Systems Section}\\
\textit{Jet Propulsion Laboratory, Pasadena,
CA 91109-8099, USA}\\
\textit{Email: {oktay.arslan@jpl.nasa.gov}}\\~\\ 
Panagiotis Tsiotras\\
\textit{Daniel Guggenheim School of Aerospace Engineering}\\ 
\textit{Institute for Robotics and Intelligent Machines}\\ 
\textit{Georgia Institute of Technology, Atlanta, GA 30332-0150,
USA}\\
\textit{Email: {tsiotras@gatech.edu}}}

\date{September 14, 2015\\[3pt] Revised June 20, 2016}

\begin{document}

\maketitle
\vspace*{-5mm}
\begin{abstract}
Recent progress in randomized motion planners has led to the development of a new class of sampling-based algorithms that provide asymptotic optimality guarantees, notably the \AlgRRTstar{} and the \AlgPRMstar{} algorithms. 
Careful analysis reveals that the so-called ``rewiring'' step in these algorithms can be interpreted as a local policy iteration (PI) step (i.e., a local policy evaluation step followed by a local policy improvement step) so that \textit{asymptotically}, as the number of samples tend to infinity, both algorithms converge to the optimal path almost surely (with probability 1).
Policy iteration, along with value iteration (VI) are common methods for solving dynamic programming (DP) problems. 
Based on this observation, recently, the \AlgRRTsharp{} algorithm has been proposed, which performs, during each iteration, Bellman updates (aka``backups'') on those vertices of the graph that have the potential of being part of the optimal path (i.e., the ``promising'' vertices).
The \AlgRRTsharp{} algorithm thus utilizes dynamic programming ideas 
and implements them incrementally on randomly generated graphs to obtain high quality solutions.
In this work, and based on this key insight, we explore a different class of dynamic programming algorithms for solving shortest-path problems 
on random graphs generated by iterative sampling methods.
These class of algorithms utilize policy iteration instead of value iteration, and thus are better suited for massive parallelization.
Contrary to   the \AlgRRTstar{} algorithm, the policy improvement during the rewiring step is not performed only locally but rather on a set of vertices that are
classified as ``promising'' during the current iteration.
This tends to speed-up the whole process.
The resulting algorithm, aptly named Policy Iteration-\AlgRRTsharp{}  (PI-\AlgRRTsharp{}) is the first of a new class of DP-inspired algorithms  for randomized motion planning that utilize PI methods.
\end{abstract}

\section{Introduction} \label{section:introduction}

Robot motion planning is one of the fundamental problems in robotics.
It poses several challenges due to the high-dimensionality of the (continuous) search space, the complex geometry of unknown infeasible regions, the possibility of a continuous action space, and the presence of differential constraints~\cite{reif1979complexity}.
A commonly  used approach to solve this problem is to form a graph by uniform or non-uniform discretization of the underlying continuous search space and employ one of the popular graph-based search-based methods (e.g., \AlgAstar{}, Dijkstra) to find a low-cost \textit{discrete} path between the initial and the final points. These approaches essentially work as abstractions of the underlying problem and hence the quality of the solution depends on the level and fidelity of the underlying abstraction.
Despite this drawback, representing the robot motion planning problem as a graph search problem has several merits, especially since heuristic graph search-based methods provide strong theoretical guarantees such as completeness, optimality, or bounded suboptimality~\cite{dijkstra1959note,hart1968formal,pearl1984heuristics}.
Unfortunately, graph search methods are not scalable to high-dimensional problems since the resulting graphs typically have an exponentially large number of vertices as the dimension of the problem increases.
This major shortcoming of grid-based graph search planners has resulted in the development of randomized (i.e., probabilistic, sampling-based) planners
that do not construct the underlying search graph a priori.
Such planners, notably, \AlgPRM{}~\cite{Kavraki1996} and \AlgRRT{}~\cite{lavalle2000rapidly,lavalle2001randomized,lavalle2006planning},
have been proven to be successful in solving many high-dimensional real-world planning problems.
These planners are simple to implement, require less memory, work efficiently for high-dimensional problems, but
come with a relaxed notion of completeness, namely, \textit{probabilistic completeness}. 
That is, the probability that the planner fails to return a solution, if one exists, decays to zero as the number of samples approaches infinity. However, the resulting paths could be arbitrarily suboptimal~\cite{karaman2011sampling}.

Recently, asymptotically optimal variants of these algorithms, such as \AlgPRMstar{} and \AlgRRTstar{} have been proposed~\cite{karaman2011sampling},
in order to remedy the undesirable behavior of the original \AlgRRT{}-like algorithms.
The seminal work of \cite{karaman2011sampling} has sparked a renewed interest to
asymptotically optimal probabilistic, sampling-based motion planners.
Several variants have been proposed that utilize the original ideas of \cite{karaman2011sampling}; a partial list includes~\cite{arslan2013useofrelaxation,arslan2015dynamicprogramming,arslan2015phdthesis,janson2015fast,otte2015mathrm}.

Careful analysis on the prototypical asymptotically optimal motion planning algorithm, that is,  the \AlgRRTstar{} algorithm,
reveals that its optimality guarantees are the result of (at first glance) hidden ideas based on dynamic programming principles.
In hindsight, this is hardly surprising.
It should be noted however,  that the explicit connection between asymptotically optimal, sampling-based algorithms operating on incrementally constructed random graphs and dynamic programming is not immediately obvious; indeed, in \AlgRRTstar{}  there is no mention of value functions or Bellman updates, although the local ``rewiring'' step in the  \AlgRRTstar{} algorithm can be interpreted as a local policy improvement step on the underlying random graph (more details about this observation are given in Section~\ref{section:random_geometric_graphs}).

Based on this key insight, one can therefore draw from the rich literature of dynamic programming and reinforcement learning in order to design suitable algorithms that compute optimal paths on incrementally constructed random graphs.
The recently proposed \AlgRRTsharp{} algorithm, for instance, utilizes a Gauss-Seidel version of asynchronous value
iteration~\cite{arslan2013useofrelaxation,arslan2015dynamicprogramming} to speed up the convergence of \AlgRRTstar{}.
Extensions based on similar ideas as  the \AlgRRTsharp{} algorithm include the FMT* algorithm~\cite{janson2015fast}, the RRT$^\mathrm{x}$ algorithm~\cite{otte2015mathrm}, and the BIT* algorithm~\cite{gammell2015batch}.
All the previous algorithms use Bellman updates (equivalently, value iterations) to propagate the cost-to-come or cost-to-go for each vertex in the graph. Vertices are ranked according to these values (along with, perhaps, an additional heuristic) and are put into a queue. The order of vertices to be relaxed (i.e., participate in the Bellman update) are chosen from the queue.
Different orderings of the vertices in the queue result in different variations of the same theme, but all of these algorithms -- one way of another -- perform a
form of asynchronous value iteration.

%It is very well-known that dynamic programming gives a nice characterization of optimal solutions for some sequential decision making problems, i.e., those of optimal substructure. One can use value iteration (VI) or policy iteration (PI) algorithms to compute optimal policies for such problems.
%However, these algorithms usually are run over a discrete state and action space which prevents us applying these methods directly on robot motion planning problem.
%Therefore, instead of using a uniform discretization of the state space, we used random graphs to get a non-uniform discretization scheme with some optimality guarantee and used a dynamic programming method to extract optimal solutions.

In this work, we depart from the previous VI-based algorithms and we propose, instead, a novel class of algorithms based on policy-iteration (PI). Some preliminary results were presented in~\cite{arslan2015dynamicprinciples}. Policy iteration is an alternative to value iteration for solving dynamic programming problems and fits naturally into our framework, in the sense that a policy in a graph search amounts to nothing more but an assignment of a (unique) parent to each vertex.
Use of policy iteration has the following benefits: first, no queue is needed to keep track of the cost of each vertex. 
A subset of vertices is selected for Bellman updates, and policy improvement on these vertices can be done in parallel at each iteration. 
Second, for a given graph, determination of the \textit{optimal} policy is obtained after a \textit{finite} number of iterations since the policy space is finite~\cite{Bertsekas2000}.
The determination of the optimal value for each vertex, on the other hand, requires an infinite number of iterations.
More crucially, and in order to find the optimal policy, only the correct ordering of the vertices is needed, not their exact value.
This can be utilized to develop approximation algorithms that speed up convergence.
Third, although policy iteration methods are often slower than value iteration methods, they tend to be better amenable for parallelization and are faster if the
structure of the problem is taken into consideration during implementation.

%The proposed algorithm (PI-\AlgRRTsharp{}) leverages two main mathematical tools to solve single-query robot motion planning problems: random graphs
%in order to obtain a sparse representation of the underlying continuous search space via a modified version of the Rapidly-exploring Random Graph (\AlgRRG) algorithm; and dynamic programming in order to efficiently compute the optimal solutions encoded in the underlying incrementally growing graph.

\section{Problem Formulation and Notation} \label{section:formulation}

Let $\X$ denote the configuration (search) space, which is assumed to be an open subset of $\reals^{d}$, where $d \in \naturals$ with $d \geq 2$.
The \emph{obstacle region} and the \emph{goal region} are denoted by $\Xobs$ and $\Xgoal$, respectively, both assumed to be closed sets.
The obstacle-free space is defined by $\Xfree = \X \setminus \Xobs$.
Elements of $\X$ are the states (or configurations) of the system.
Let the \emph{initial configuration} of the robot be denoted by $\xinit \in \Xfree$.
The (open) neighborhood of a state $x \in \X$ is the open ball of radius $r > 0$ centered at $x$, that is, $B_{r}(x) = \{ x^{\prime} \in \X : \|x - x^{\prime} \| < r \}$.
Given a subset $S \subseteq V$ the notation $|S|$ is its cardinality, that is, the number of elements of $S$.

We will approximate $\Xfree$ with an increasingly dense sequence of discrete subsets of $\Xfree$.
That is, $\Xfree$ will be approximated by a finite set of configuration points selected randomly from $\Xfree$.
Each such discrete approximation of  $\Xfree$ will be encoded in a graph $\graph{G} = (V, E)$ with $V$ being the set of vertices (the elements of the discrete approximation of $\Xfree$) and with edge set $E \subseteq V \times V$ encoding allowable transitions  between elements of $V$.
Hence, $\graph{G}$ is a directed graph.
Transitions between two vertices $x$ and $x'$  in $V$ are enabled by a control action $u \in U(x)$
such that $x'$ is the successor vertex of $x$  in $\graph{G}$ under the action $u$ so that $(x,\succx) \in E$.
Let $U = \cup_{x \in V} U(x)$.
We use the mapping $ f : V \times U \rightarrow V$ given by
\begin{equation} \label{eq100}
\succx= f(x,u), \qquad u \in U(x),
\end{equation}
to formalize the transition from $x$ to $x'$ under the control action $u$.
In this case, we say that $\succx$ is the successor of $x$ and that $x$ is the predecessor of $\succx$.
The set of predecessors of $x\in V$ will be denoted by $\PrcPredecessor(\graph{G},x)$, and the  set of
successors of $x$ will be denoted by $\PrcSuccessor(\graph{G},x)$.
Also, we let  $\overline{\PrcPredecessor}(\graph{G},x) =  \PrcPredecessor(\graph{G},x) \cup \{x\}$.
Note that, using the previous definitions, the set of admissible control actions at $x$ may be equivalently defined as
\begin{equation}
U(x) = \{u:  \succx = f(x,u), ~ \succx \in \PrcSuccessor(\graph{G},x) \}.
\end{equation}
Thus, the control set $U(x)$ defines unambiguously the set $\PrcSuccessor(\graph{G},x)$ of the successors of $x$, in the sense that
there is one-to-one correspondence between
control actions $u\in U(x)$ and elements of $\PrcSuccessor(\graph{G},x)$ via (\ref{eq100}).
Equivalently, once the directed graph $\graph{G}$ is given, for each edge $(x,\succx) \in E$ corresponds a control $u\in U(x)$ enabling this transition.
It should be remarked that the latter statement, when dealing with dynamical systems (such as robots, etc) amounts to a controllability condition.
Controllability is always satisfied for fully actuated systems, but may not be satisfied for underactuated systems (such as for many case of kinodynamic planning with differential constraints).
For sampling-based methods such as $\AlgRRTstar{}$ this controllability condition is equivalent to the existence of a steering function that drives the system between any two given states.

Once we have abstracted $\Xfree$ using the graph $\graph{G}$, the motion planning problem becomes one of a shortest path problem on the graph
$\graph{G}$.
To this end, we define the path $\sigma$ 
in $\graph{G}$ to be a sequence of vertices $\sigma = (x_0,x_1,\ldots,x_N)$ such that $x_{k+1} \in \mathtt{succ}(\mathcal{G},x_k)$ for all $k = 0,1,\ldots,N-1$.
The length of the path is $N$, denoted by $\duration(\sigma) = N$.
When we want to specify explicitly the first node of the path we will use the first node as an argument, i.e., we will write $\sigma(x_0)$.
The $k$th element of $\sigma$ will be denoted by $\sigma_k$. That is,  if $\sigma(x_0) = (x_0,x_1,\ldots,x_N)$ then $\sigma_k(x_0) = x_k$ for all $k=0,1,\ldots,N$.
A path is rooted at $\xinit$ if $x_0 = \xinit$.
A path rooted at $\xinit$ terminates at a given goal region $\Xgoal \subset \Xfree$ if $x_N \in \Xgoal$.

To each edge $(x,\succx)$ 
encoding an allowable transition from $x \in \Xfree$ to $\succx \in \PrcSuccessor(\graph{G},x)$, we associate a finite cost $\costValue(x, \succx)$.
Given a path $\sigma(x_0)$, the cumulative cost along this path is then
\begin{equation}
\sum_{k=0}^{N-1} \costValue(x_{k}, x_{k+1}).
\end{equation}

Given a point $x \in \X$, a mapping $\mu: x \mapsto u \in U(x)$ that assigns a control action to be executed at each
point $x$ is called a \textit{policy}. 
Let $\mathcal{M}$ denote the space of all policies.
Under some assumptions on the connectivity of the graph $\graph{G}$ and the cost of the directed edges, one can use DP algorithms and the corresponding Bellman equation in order to compute optimal policies.
%These assumptions assert if the shortest path problem is well posed, and  yield existence of proper policies or improper policies with finite cost.
Note that a policy $\mu \in \mathcal{M}$ for this problem defines a graph whose edges are $(x,f(x,\mu(x))) \in E$ for all $x \in V$. The policy $\mu$ is proper if and only if this graph is acyclic, i.e., the graph has no cycles.
Thus, there exists a proper policy $\mu$ if and only if each node is connected to the $\Xgoal$ with a directed path.
Furthermore, an improper policy has finite cost, starting from every initial state, if and only if all the cycles of the corresponding graph have non-negative cost~\cite{Bertsekas2000}.
Convergence of the DP algorithms is proven if the graph is connected and the costs of all its cycles are positive~\cite{bertsekas2013abstract}.

\section{Overview of Dynamic Programming}

Dynamic programming solves sequential decision-making problems having a finite number of stages.
In terms of DP notation, our system has the following equation
\begin{equation} \label{eq:pt200}
 \succx = f(x,u)
\end{equation}
where the cost function is defined as
\begin{equation} \label{eq:pt201}
g(x,u) = \costValue(x, f(x,u)).
\end{equation}
Given a sequential decision problem of the form (\ref{eq:pt200})-(\ref{eq:pt201}),
it is well known that the optimal cost function satisfying the following \textit{Bellman equation}:
\begin{equation} \label{eqn:deterministic_system:bellmans_equation}
J^{*} (x) = \inf_{u \in U(x)} \Bigl\{ g(x,u) +  J^{*} (f(x,u)) \Bigr\}, \quad \forall x \in \X.
\end{equation}
The result of the previous optimization results in an optimal policy
$\mu^{*} \in \mathcal{M}$, that is,
\begin{equation} \label{eqn:deterministic_system:optimal_policy}
\mu^{*} (x) \in \argmin_{ u \in U(x)} \Bigl\{ g(x,u) +  J^{*} (f(x,u))\Bigr\}, \quad \forall x \in \X.
\end{equation}
Note that if we are given a policy $\mu \in \mathcal{M}$ (not necessarily optimal) we can compute its cost from
\begin{equation} \label{eqn:deterministic_system:bellmans_equation_mu}
J_\mu (x) =  g(x,\mu(x)) +  J_{\mu} (f(x,\mu(x))), \quad \forall x \in \X.
\end{equation}
It follows that $J^{*}(x) = \inf_{\mu \in \mathcal{M}} J_{\mu}(x), \quad x \in \X$.
By introducing  the  expression
\begin{equation} \label{def:Hfun}
H(x,u,J) = g(x,u) +  J( f(x,u)), \quad x \in \X, u \in U(x).
\end{equation}
and letting the operator $T_{\mu}$ for a given policy $\mu \in \mathcal{M}$,
\begin{equation}
(T_{\mu} J) (x) = H (x, \mu(x), J), \quad x \in \X,
\end{equation}
we can define the Bellman operator $T$
\begin{equation}
(TJ)(x) = \inf_{ u \in U(x)} H(x,u,J) = \inf_{\mu \in \mathcal{M}} (T_{\mu} J) (x), \quad x \in \X,
\end{equation}
which allows us to write the Bellman equation (\ref{eqn:deterministic_system:bellmans_equation}) succinctly as follows
\begin{equation}
J^{*} = TJ^{*},
\end{equation}
and the optimality condition (\ref{eqn:deterministic_system:optimal_policy}) as
\begin{equation}
T_{\mu^{*}} J^{*} = T J^{*}.
\end{equation}
This interpretation of the Bellman equation states that $J^{*}$ is the fixed point of the Bellman operator $T$, viewed as a mapping from the set of real-valued functions on $\X$ into itself. Also, in a similar way, $J_{\mu}$, the cost function of the policy $\mu$, is a fixed point of $T_{\mu}$ (see (\ref{eqn:deterministic_system:bellmans_equation_mu})).

There are three different classes of DP algorithms to compute the optimal policy $\mu^{*}$ and the optimal cost function~$J^{*}$.
%\begin{enumerate}
%\item Value Iteration
%\item Policy Iteration
%\item Optimistic Policy Iteration
%\end{enumerate}

\paragraph{Value Iteration (VI).}

This algorithm computes $J^{*}$ by relaxing Eq.~(\ref{eqn:deterministic_system:bellmans_equation}), starting with some $J^{0}$, and generating a sequence $\bigl\{T^{k}J \bigr\}_{k=0}^\infty$ using the iteration
\begin{equation}
J^{k+1} = T J^{k}
\end{equation}
The generated sequence converges to the optimal cost function due to contraction property of the Bellman operator $T$~\cite{bertsekas2013abstract}. 
This method is an indirect way of computing the optimal policy $\mu^{*}$, using the information of the optimal cost function~$J^{*}$.

\paragraph{Policy Iteration (PI).}

This algorithm starts with an initial policy  $\mu^{0}$ and generates a sequence of  policies $\mu^{k}$ by performing Bellman updates. Given the current policy $\mu^{k}$, the typical iteration is performed in two steps:

\begin{enumerate}[i)]
\item \textbf{Policy evaluation}: compute $\Jmuk$ as the unique solution of the equation
\begin{equation} \label{eq:PIdef}
\Jmuk = \Tmuk \Jmuk.
\end{equation}

\item \textbf{Policy improvement}: compute a policy $\mu^{k+1}$ that satisfies
\begin{equation}
T_{\mu^{k+1}} \Jmuk = T \Jmuk.
\end{equation}
\end{enumerate}

\paragraph{Optimistic Policy Iteration (O-PI).}

This algorithm works the same as PI, but differs in the policy evaluation step. 
Instead of solving the system of linear equations exactly in the policy evaluation step (\ref{eq:PIdef}), it performs an approximate evaluation of the current policy and uses this information in the subsequent policy improvement step.

\begin{comment}
\begin{description}
\item[Optimal Motion Planning]\label{problem:omp} Given a bounded and connected open set $\X \subset \reals^{{d}}$, the sets $\Xfree$ and $\Xobs = \X \backslash \Xfree$, and  an initial point $\xinit \in \Xfree$ and a goal region $\Xgoal \subset \Xfree$, find the minimum-cost path connecting $\xinit$ to the goal region $\Xgoal$. %If no such path exists, then report that no solution is possible.

\item[Optimal Feedback Motion Planning]\label{problem:ofmp} Given a bounded and connected open set $\X \subset \reals^{{d}}$, the sets $\Xfree$ and $\Xobs = \X \backslash \Xfree$, and a goal region $\Xgoal \subset \Xfree$, find the minimum-cost path connecting any $x \in \Xfree$ to the goal region $\Xgoal$.
\end{description}
\end{comment}

\section{Random Geometric Graphs}\label{section:random_geometric_graphs}

The main difference between standard shortest path problems on graphs and sampling-based methods for solving motion planning problems is the fact that in the former case the graph is given a priori, whereas in the latter case the path is constructed on-the-fly by sampling randomly allowable configuration points from $\Xfree$ and by constructing the graph $\graph{G}$ incrementally, adding one, or more, vertices at each iteration step. 
Of course, such an iterative construction raises several questions, such as: is the resulting graph connected? under what conditions one can expect that $\graph{G}$ is an accurate representation of $\Xfree$? 
how does discretizing the actions/control inputs affects the movement between sampled successor vertices, etc. 
All these questions have been addressed in a series of recent papers~\cite{karaman2010optimal,karaman2011sampling} so we will not elaborate further on the graph construction.
Suffice it to say, such random geometric graphs (RGGs) can be constructed easily and such graphs have been the cornerstone of the recent emergence of asymptotically optimal sampling based motion planners.

For completeness, and in order to establish the necessary connections between DP algorithms and RRGs, we provide a brief overview of random graphs as they are used in this work.
For more details, the interested reader can peruse~\cite{Bertsekas2000} or \cite{penrose2003random}.

In graph theory, a random geometric graph (RGG) is a mathematical object that is usually used to represent spatial networks. RGGs are constructed by placing a collection of vertices drawn randomly according to a specified probability distribution. These random points constitute the node set of the graph in some topological space. Its edge set is formed via pairwise connections between these nodes if certain conditions (e.g., if their distance according to some metric is in a given range) are satisfied. Different probability distributions and connection criteria yield random graphs of different properties.

An important class of random geometric graphs is the \textit{random r-disc graphs}. Given the number of points $n$ and a nonnegative radius value $r$, a random $r$-disc graph in $\reals^{d}$ is constructed as follows: first, $n$ points are independently drawn from a uniform distribution. These points are pairwise connected if and only if the distance between them is less than $r$. Depending on the radius, this simple model of random geometric graphs possesses different properties as the number of nodes $n$ increases.
A natural question to ask is how the connectivity of the graph changes for different values of the connection radius as the number of samples goes to infinity. In the literature, it is shown that the connectivity of the random graph exhibits a phase transition, and a connected random geometric graph is constructed almost surely when the connection radius $r$ is strictly greater than a critical value $r^{*} = \big\{ \log(n)/ (n \zeta_{d}) \big\}^{d}$, where $\zeta_{d}$ is volume of the unit ball in $\reals^{d}$.
If the connection radius is chosen less than the critical value $r^{*}$, then, multiple disconnected clusters occur almost surely as $n$ goes to infinity~\cite{penrose2003random}.

Recently, novel connections have been made between motion planning algorithms and the theory of random geometric graphs~\cite{karaman2011sampling}. These key insights have led to the development of a new class of algorithms which are asymptotically optimal (e.g., \AlgRRG, \AlgRRTstar, \AlgPRMstar).
For example, in the \AlgRRG{} algorithm, a random geometric $r$-disc graph is first constructed incrementally for a fixed number of iterations. 
Then, a post-search is performed on this graph to extract the encoded solution. The key step is that the connection radius is shrunk as a function of vertices, while still being strictly greater than the critical radius value.
By doing so, it is guaranteed to obtain a connected and sparse graph, yet the graph is rich enough to provide asymptotic optimality guarantees, almost surely. 
The authors in \cite{karaman2011sampling} showed that the \AlgRRG{} algorithm yields a consisted discretization of the underlying continuous configuration space, i.e., as the number of points goes to infinity, the lowest-cost solution encoded in the random geometric graph converges to the optimal solution embedded in the continuous configuration space with probability one. In this work, we leverage this nice feature of random geometric graphs to get a consistent discretization of the continuous domain of the robot motion planning problem. With the help of random geometric graphs, the robot motion planning problem boils down to a shortest path problem on a discrete graph.

\section{Proposed Approach}

\subsection{From RRGs to DP}

Let $\graph{G} = (V, E)$ denote the graph constructed by the \AlgRRG{} algorithm at some iteration, where $V$ and $E \subseteq V \times V$ are finite sets of vertices and edges, respectively.
Based on the previous discussion, $\graph{G}$ is connected and all edge costs are positive, which implies that the cost of all the cycles in $\graph{G}$ are positive.
Using the notation introduced in Section~\ref{section:formulation} , we can define on this graph the sequential decision system (\ref{eq:pt200}) where  $\succx \in \PrcSuccessor(\graph{G},x)$ and with transition cost as in (\ref{eq:pt201}).
Once a policy $\mu$ is given (optimal or not), there is a unique $\succx \in \PrcSuccessor(\graph{G},x)$ such that $\succx = f(x,\mu(x))$, called the parent of $x$.
Accordingly, $x$ is the child of $\succx$ under the policy $\mu$. 
Conversely, a parent assignment for each node in $\graph{G}$ defines a policy.
Note that each node has a single parent under a given policy, but may have multiple children.

In our case, the graph computed by the \AlgRRG{} algorithm is a connected graph by construction, and all edge cost values are positive, which implies that the costs of all its cycles are positive. Therefore, convergence is guaranteed and the resulting optimal policy is proper.

\subsection{DP Algorithms for Sampling-based Planners}

The sampling-based motion planner which utilizes VI, i.e., \AlgRRTsharp{}, was presented in~\cite{arslan2013useofrelaxation}. The \AlgRRTsharp{} algorithm implements the Gauss-Seidel version of the VI algorithm and provides a sequential implementation. In this work, we follow up on the same idea and propose a sampling-based algorithm which utilizes PI algorithm as shown in Figure~\ref{fig:rrtsharp_pi_flow_diagram}.

\begin{figure}
% l b r t
% {
% \setlength{\fboxsep}{0pt}%
% \setlength{\fboxrule}{1pt}%
% \fbox{}
% }
\centering
\scalebox{0.4}{\includegraphics[]{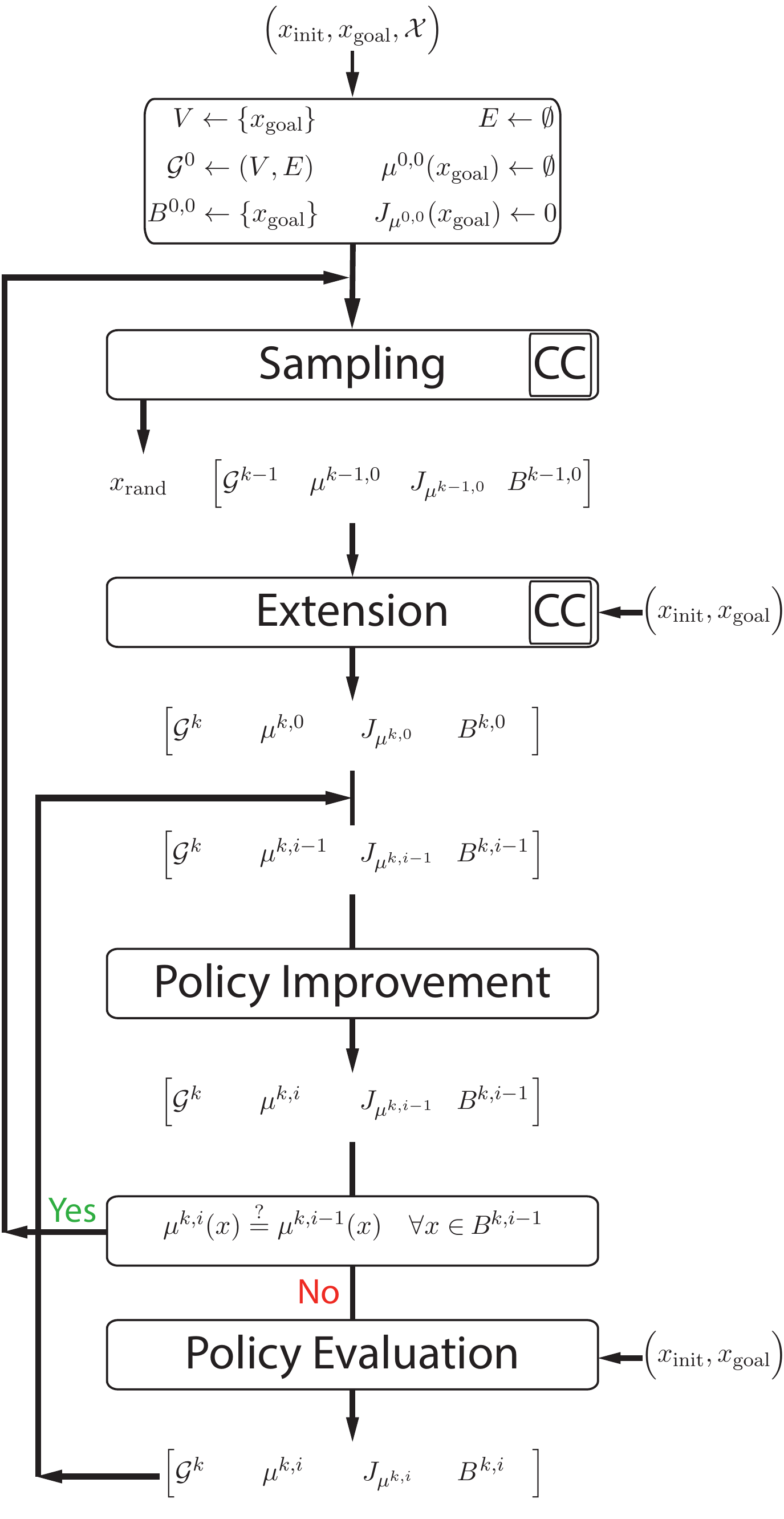}}
\caption{Overview of the PI-\AlgRRTsharp{} Algorithm}\label{fig:rrtsharp_pi_flow_diagram}

\end{figure}

The body of PI-\AlgRRTsharp{} algorithm is given in Algorithm~\ref{alg:rrtsharp_omp_pi}. The algorithm initializes the graph with $\xgoal$ in Line 2 and incrementally builds the graph from $\xgoal$  toward $\xinit$. The algorithm includes a new vertex and a couple new edges into the existing graph at each iteration. If this new information has a potential to improve the existing policy, then, a slightly modified PI algorithm is called subsequently in the $\PrcReplan$ procedure. 
Specifically, and for the sake of numerical efficiency, unlike the standard PI algorithm, policy improvement is performed only for a subset vertices $B$ which have the potential to be part of the optimal solution. As shown in Figure~\ref{fig:rrtsharp_pi_flow_diagram}, $B^{k,i}$ is the set of these vertices during the $k$th iteration and $i$th policy improvement step. The fact that this modification of the PI still ensures the asymptotic optimality, almost surely, of the proposed PI-\AlgRRTsharp{} algorithm requires extra analysis, 
which is presented in Section~\ref{sec:analysis}.

\input{rrtsharp_omp_pi.tex}

The $\PrcExtend$ procedure is given in Algorithm~\ref{alg:extend_rrtsharp_omp_pi}. If a new vertex is decided for inclusion, its control is initialized by performing policy improvement in Lines 6-14. Then, it is checked in Line 15 if the new vertex has a potential to improve the existing policy where $h$ denotes an admissible heuristic function that computes an estimate of the cost between two given points. 
If so, it is included to the set of vertices which are selected to perform policy improvement.
Such heuristics, have been previously used to focus the search of sampling-based planner, see for example~\cite{otteieeetro13,arslan2013useofrelaxation}.

\input{extend_rrtsharp_omp_pi.tex}

The $\PrcReplan$ procedure which implements the PI algorithm is shown in Algorithm~\ref{alg:replan_omp_pi}. 
The policy improvement step is performed in Lines 2-9 until the cost of the existing policy
becomes almost stationary. 
Note that the for-loop in the $\PrcReplan$ procedure can run in parallel.

\input{replan_omp_pi.tex}

The policy evaluation step is implemented in Algorithm~\ref{alg:omp_policy_evaluation}. Algorithm~\ref{alg:omp_policy_evaluation} solves a system of linear equations by exploiting the underlying structure. Simply, the existing policy forms a tree in the current graph and the solution of the system of linear equations corresponds to the cost of each path connecting vertices to the goal region via edges of the tree. 
If $\xinit$ is already in the graph, the algorithm computes the cost of the path between $\xinit$ and $\xgoal$ by using queue $q$.  
Subsequently, the set of vertices that are promising, i.e., those in the set $B$ and their cost-to-go values are computed by using the cost-to-go value of $\xinit$.

\input{omp_policy_evaluation.tex}

\FloatBarrier
%\newpage
\section{Theoretical Analysis} \label{sec:analysis}

The main purpose of this section is to show that the proposed PI-\AlgRRTsharp{} algorithm inherits the nice properties of the \AlgRRG{} and \AlgRRTstar{} algorithms and thus it is asymptotically optimal, almost surely.
The result follows trivially if policy iteration is performed on all the vertices of the current graph $\graph{G}^{k}$ at the $k$th iteration of the PI-\AlgRRTsharp{} algorithm. 
However, this will involve performing policy iteration also on vertices that may not have the potential to be part of the optimal solution.
Ideally, and for the sake of numerical efficiency, we wish to preform policy improvement only on those vertices that have the potential of being part of the optimal solution, and only those.
This will require a more detailed analysis, since we only have an estimate of this set of (so-called promising) vertices.
The basic idea of the proof is based on the fact that each policy improvement progresses sequentially and computes best paths of length one, then of length two, then of length three, and so on. 
This allows us to keep track of all promising vertices that can be part of the optimal path of increasing lengths (e.g., increasing number of path edges) as we start from the goal vertex and move back towards the start vertex.
The proof is rather long and  thus is is split in a sequence of several lemmas. 

To this end, let $\graph{G}^{k} = (V^{k},E^{k})$ denote the graph at the end of the $k$th iteration of the PI-\AlgRRTsharp{} algorithm. Given a vertex $x \in V^{k}$, let the control set $U^{k}(x)$ be divided into two disjoint sets $U^{k,*}(x)$ and $U^{k,\prime}(x)$ and  and let the successor set $\PrcSuccessor(\graph{G}^{k},x)$ also  be divided accordingly into two disjoint sets $S^{k,*}(x)$ and $S^{k,\prime}(x)$, as follows:
\begin{itemize}
\item $U^{k}(x) = U^{k,*}(x) \cup U^{k,\prime}(x)$ where $U^{k,*}(x) = \arg \min\limits_{u \in U^{k}(x)} H(x,u,J^{k,*})$ and $U^{k,\prime}(x) = U^{k}(x) \setminus U^{k,*}(x)$
\item $\PrcSuccessor(\graph{G}^{k},x) = S^{k,*}(x) \cup S^{k,\prime}(x)$ where $S^{k,*}(x) = \{ x^{*} \in V^{k} : \exists u^{*} \in U^{k,*}(x) \text{ s.t. } x^{*} = f(x, u^{*}) \}$ and $S^{k,\prime}(x) = \PrcSuccessor(\graph{G}^{k},x) \setminus S^{k,*}(x)$
\end{itemize}

Let $\mathcal{M}^{k}$ denote the set of all policies at the end of the $k$th iteration of the PI-\AlgRRTsharp{} algorithm, that is, let $\mathcal{M}^{k} = \{ \mu: \mu(x) \in U^{k}(x), \, \forall x \in V^{k} \}$. 
Given a policy $\mu \in \mathcal{M}^{k}$ and an initial state $x$, let $\sigma^\mu(x)$ denote the path resulting from executing the policy $\mu$ starting at $x$.
That is, $\sigma^\mu(x) = (x_0,x_1,\ldots, x_N)$ such that $\sigma_0^\mu(x)  = x$ and  $\sigma_N^\mu(x)  \in \Xgoal$ for $N > 0$, where $\sigma_j^\mu(x)$ is the $j$th element of $\sigma^\mu(x)$. By definition,  $x_{j+1} = f(x_j,\mu(x_{j}))$ for $j = 0, 1, \ldots, N-1$.
%
%Note that $\sigma^\mu(x)$ satisfies the following properties:
%\begin{itemize}
%\item $\sigma^\mu_0 (x) = x$
%\item $\sigma^\mu_N(x) \in \Xgoal$ for $N > 0$
%\item $\sigma^\mu_j (x)= x_{j} \text{ such that } x_{j+1} = f(x_j,\mu(x_{j}))$ for $j = 0, 1, \ldots, (N-1)$
%\end{itemize}
%where $\sigma_j^\mu(x)$ is the $j$th element of $\sigma^\mu(x)$.
%
Let now $\Sigma^{k} (x)$ be the set of all paths rooted at $x$ and let
$\Sigma^{k,*} (x)$ denote the set of all lowest-cost paths rooted at  $x$ that reach the goal region at the $k$th iteration of the algorithm, that is,
\begin{equation*}
\Sigma^{k,*} (x) = \{ \sigma \in \Sigma^{k} (x) : \sigma = \sigma^\mu(x) ~\text{such that}~ \exists \mu \in \mathcal{M}^{k} ,~J_{\mu} (x) = J^{k,*} (x) \}.
\end{equation*}

%Let $\duration(\sigma^\mu(x)) = N$ denote the length of $\sigma^\mu(x)$. 
Note that the set $\Sigma^{k,*} (x) $ may contain more than a single path.
Finally, let  $N^{k} (x)$ denote the  shortest path length in  $\Sigma^{k,*} (x) $, that is,
\[
N^{k} (x) = \min\limits_{ \sigma^\mu(x) \in \Sigma^{k,*} (x) } \duration(\sigma^\mu(x)) 
\]

%%%$$
%%\Big[\mathcal{G}^{k} \quad \mu^{k,i-1} \quad J_{\mu^{k,i-1}}  \quad B_{k,i-1}\Big]
%%$$

Let us define the following sets for a given policy $\mu^{k,i} \in \mathcal{M}^k$ and its corresponding value function $\Jmuki$ at the end of the $i$th policy improvement step and at the $k$th iteration of the PI-\AlgRRTsharp{} algorithm:

\begin{enumerate}[a)]

\item 
The set of vertices in $\graph{G}^{k}$  whose optimal cost value is less than that of $\xinit$,
$$
\Vprom^{k} = \{ x \in V^{k} :  J^{k,*} (x) < J^{k,*} (\xinit) \}
$$
This is the set of promising vertices.

\item 
The set of promising vertices in $\graph{G}^{k}$, whose optimal cost value is achieved by executing the policy $\mu^{k,i}$ at the $i$th policy iteration step
$$
O^{k,i} = \{ x \in V^{k} : \Jmuki(x) = J^{k,*} (x) < J^{k,*} (\xinit)\}
$$ 

\item 
The set of vertices in $\graph{G}^{k}$  that can be connected to the goal region at iteration $k$ with an optimal path of length less than or equal to $\ell$
$$
L^{k,\ell} = \{ x \in V^{k} : \exists \sigma^\mu (x) \in \Sigma^{k,*}(x) ~\mathrm{s.t.}~ N^{k} (x) \leq \ell \}
$$

\item The set of promising vertices  in $\graph{G}^{k}$  that are connected to the goal region via optimal paths whose length is less than or equal to $\ell$ 
$$
P^{k,\ell} = L^{k,\ell} \cap \Vprom^{k} 
%= \{ x \in V^{k} : N^{k} (x) \leq \ell ~\textrm{and}~ J^{k,*} (x) < J^{k,*} (\xinit)\}
$$

%%$$
%%P^{k,i} = \{ x \in V^{k} : t_{\mathrm{f}}^{*} (x) \leq i, J^{k,*}(x) < J^{k,*}(\xinit) \} \text{ and } \partial P^{k,i} = P^{k,i+1} \setminus P^{k,i}
%%$$ 

\item The set of vertices that are selected for a Bellman update during the beginning of the $i$th policy improvement
$$
B^{k,i} = \{ \overline{\PrcPredecessor}(\graph{G}^{k}, x) : x \in V^{k} , ~ J_{{\mu}^{k,i}}(x) < J_{{\mu}^{k,i}} (\xinit) \}
$$
\end{enumerate}

Note from d)  that the set of promising vertices that can be connected to the goal region via optimal paths whose length is exactly $\ell+1$ is given by
\[ 
 \partial P^{k,\ell} = P^{k,\ell+1} \setminus P^{k,\ell}
\]
It should also be clear from these definitions that $O^{k,i} \subseteq B^{k,i}$ for all $k = 1,2,\ldots$ and $i=0,1,2,\ldots$.  

%Let $L^{k,i} = \{ x \in V^{k} :  \exists \mu \in \mathcal{M}^{k} \text{ s.t. } J_{\mu} (x) = J^{k,*} (x), \, \duration(\sigma(\cdot; x, \mu)) \leq i \}$.

\begin{lemma}\label{lemma:nondecreasing_set_vertices_of_optimal_cost_value}
The sequence $O^{k,i}$ generated by the policy iteration step of the \emph{PI}-\AlgRRTsharp{} algorithm is non-decreasing, that is, 
$O^{k,i} \subseteq O^{k,i+1}$ for all $i = 0,1,\ldots$.
\end{lemma}

\begin{proof}
First, note that $O^{k,0} = V^k \cap \Xgoal \neq \varnothing$.
%The policy  has the following property:
Let now $i  > 0$ and assume that $x \in O^{k,i}$.  
By definition, we have that $J_{{\mu}^{k,i}} = T_{{\mu}^{k,i}} J_{{\mu}^{k,i}}$ 
where $\mu^{k,i}$ is the policy computed at the end of $i$th policy improvement step at the $k$th iteration of the PI-\AlgRRTsharp{} algorithm.
The previous expression implies that $ (T_{{\mu}^{k,i}} J_{{\mu}^{k,i}}) (x) = J^{k,*} (x) $, and hence 
$\mu^{k,i}(x) \subseteq U^{k,*}(x)$. Similarly, the cost function $J_{{\mu}^{k,i}}$ satisfies 
$J_{{\mu}^{k,i}} (x) = J^{k,*}(x) < J^{k,*}(\xinit) \leq J_{{\mu}^{k,i}} (\xinit) $ which yields $J_{{\mu}^{k,i}} (x) < J_{{\mu}^{k,i}} (\xinit)$.
It follows that the vertex $x$ and its predecessors will be selected for Bellman update during the next policy improvement, that is, 
$\overline{\PrcPredecessor}(\graph{G}^{k},x) \in B^{k,i}$. 

After policy improvement, the  updated policy and the corresponding cost function are given by
$$
(T_{{\mu}^{k,i+1}} J_{{\mu}^{k,i}}) (x) = (T J_{{\mu}^{k,i}}) (x) = J^{k,*}(x),
$$
which implies that $(T_{{\mu}^{k,i+1}} J_{{\mu}^{k,i}}) (x) = J^{k,*}(x)$ and hence
 $\mu^{k,i+1}(x) \subseteq U^{k,*}(x) $.
 Similarly, 
$J^{k,*}(x) \leq J_{{\mu}^{k,i+1}} (x) \leq J_{{\mu}^{k,i}} (x) = J^{k,*}(x) $, and hence
$J_{{\mu}^{k,i+1}} (x) = J^{k,*}(x) < J^{k,*} (\xinit)$. It follows that $x \in O^{k,i+1}$. \QED
\end{proof}

\begin{lemma}\label{lemma:nondecreasing_set_vertices_of_paths_with_duration}
The sequence $L^{k,\ell}$ is non-decreasing, that is, $L^{k,\ell} \subseteq L^{k,\ell+1}$ for $\ell = 0, 1, \ldots$.
Furthermore, for all $x \in \partial L^{k,\ell} = L^{k,\ell+1} \setminus L^{k,\ell}$, there exists $x^{*} \in L^{k,\ell}\cap S^{k,*}(x)$.
\end{lemma}

\begin{proof}
For $\ell = 0$ we have that $L^{k,0} = V^k \cap \Xfree \not= \varnothing$.
Let now $\ell > 0$, and assume that $x \in L^{k,\ell}$. 
Then, by definition, there exists a policy $\mu \in \mathcal{M}^{k}$ such that the vertex $x$ achieves its optimal cost function value, $J_{\mu} (x) = J^{k,*} (x)$, and the optimal path connecting $x$ to the goal region has length less  than or equal to $\ell$, that is, $\duration(\sigma^\mu(x)) \leq \ell < \ell + 1$, 
which implies, trivially, that $x \in L^{k,\ell+1}$. 
%Therefore, we have $L^{k,\ell} \subseteq L^{k,\ell+1}$. Same property also holds for the sequence $P^{k,\ell}$ which can be shown as follows:
%$$
%x \in P^{k,\ell} \: \Rightarrow \: x \in L^{k,\ell} \land x \in \Vprom^{k} \: \Rightarrow \: x \in L^{k,\ell+1} \land x \in \Vprom^{k} \: \Rightarrow \: x \in P^{k,\ell+1}
%$$

To show the second part of the statement, first notice that,
by definition, the vertices in the set $\partial L^{k,\ell}$ are the ones that can be connected to the goal region via an optimal path of length exactly $\ell+1$. Let us now assume that $x \in \partial L^{k,\ell}$ and let $\sigma^\mu(x) \in \Sigma^{k,*}(x)$ be the optimal path of length $\ell+1$ between $x$ and the goal region. Let  $\sigma_1^\mu(x) = x^{*} $ and $\sigma^{\mu}(x^{*})$ be the sub-arc rooted at  $x^{*}$ resulting from applying $\mu$. By construction of the path $\sigma^{\mu}(x)$, we have that $x \in \PrcPredecessor(\graph{G}^{k},x^{*})$. 
Also, since $\sigma^{\mu}(x)$ is the optimal path rooted at $x$, the control action applied at vertex $x$ needs to be optimal , that is, 
$\mu(x) \in U^{k,*}(x)$ and $\sigma^{\mu} (x^{*})$ is the optimal path connecting $x^{*}$ to the goal region due to the \textit{principle of optimality}, where $\sigma^{\mu} (x^{*}) \in \Sigma^{k,*}(x^{*})$, which implies that $x^{*} \in S^{k,*}(x)$.
Furthermore,  $x^{*} \in L^{k,\ell}$ since $\duration(\sigma^{\mu} (x^{*})) = \ell$. \QED
\end{proof}

\begin{corollary} \label{corollary1}
The sequence $P^{k,\ell}$ is non-decreasing, that is, $P^{k,\ell} \subseteq P^{k,\ell+1}$ for $\ell = 0, 1, \ldots$.
Furthermore, for all $x \in \partial P^{k,\ell} = P^{k,\ell+1} \setminus P^{k,\ell}$, there exists $x^{*}\in P^{k,\ell} \cap S^{k,*}(x)$.
\end{corollary}

\begin{proof}
The first part of the result follows immediately from Lemma~\ref{lemma:nondecreasing_set_vertices_of_paths_with_duration}. 
To show the second part, notice that, from the definition of the boundary set, we can rewrite $\partial P^{k,\ell}$ as follows:
$$
\partial P^{k,\ell} = P^{k,\ell+1} \setminus P^{k,\ell} =  (L^{k,\ell+1} \cap \Vprom^{k}) \setminus (L^{k,\ell} \cap \Vprom^{k}) = \partial L^{k,\ell} \cap \Vprom^{k}
$$
Let now $x \in \partial P^{k,\ell}$, which implies that  $x \in \partial L^{k,\ell}$ and $x \in \Vprom^{k}$. 
From Lemma~\ref{lemma:nondecreasing_set_vertices_of_paths_with_duration} 
there exists $x^{*} \in L^{k,\ell} \cap S^{k,*}(x)$ such that $x \in \PrcPredecessor(\graph{G}^{k}, x^{*})$. 
We need to show that $x^{*} \in P^{k,\ell}$.
Since $x^{*} \in L^{k,\ell}$  we only need to show 
that  $x^{*} \in \Vprom^{k}$. 
Since $x$ is a promising vertex, its optimal cost value satisfies $J^{k,*}(x) < J^{k,*} (\xinit)$. We know that the optimal cost function value of $x^{*}$ satisfies $J^{k,*}(x) = g(x,u^{*}) + J^{k,*} (x^{*})$ where $x^{*} = f(x,u^{*})$ and $u^{*} \in U^{k,*}(x)$. Since $g(x,u^{*})$ is nonnegative, we have that $J^{k,*}(x^{*}) \leq J^{k,*}(x) < J^{k,*} (\xinit)$ which implies $x^{*} \in \Vprom^{k}$ and hence $x^{*} \in P^{k,\ell}$.  \QED
\end{proof}

\begin{lemma}\label{lemma:conditions_for_achieving_optimal_policy_and_cost}%Given a policy $\mu^{k,i}$ and its corresponding value function $\Jmuki$, 
Let $x \in B^{k,i}$ and assume that $\Jmuki (x^{*}) = J^{k,*}(x^{*})$ where $x^{*} \in S^{k,*}(x)$.
Then $\mu^{k,i+1}(x) \subseteq U^{k,*}(x)$ and $J_{\mu^{k,i+1}}(x) = J^{k,*}(x)$ at the end of $(i+1)$th policy improvement step at the $k$th iteration of the \emph{PI}-\AlgRRTsharp{} algorithm.
%
%\begin{enumerate}[i)]
%\item If $\Jmuki (s) = J^{k,*}(s)$, then $\mu^{k,i+1}(x) \subseteq U^{k,*}(x)$ and $J_{\mu^{k,i+1}}(x) = J^{k,*}(x)$.
%
%% % % % %\item If $\Jmuki (z) = J^{k,*}(z)$ and $\Jmuki (z^{\prime}) = J^{k,*}(z^{\prime})$, then $\mu^{k,i+1}(x) \subseteq U^{k,*}(x)$ and $J_{\mu^{k,i+1}}(x) = J^{k,*}(u)$.
%
%% % % % %\item If $\Jmuki (z) = J^{k,*}(z)$ and $\Jmuki (z^{\prime}) \neq J^{k,*}(z^{\prime})$, then $\mu^{k,i+1}(x) \subseteq U^{k,*}(x)$ and $J_{\mu^{k,i+1}}(x) = J^{k,*}(u)$.
%
%\item If $\Jmuki (s) > J^{k,*}(s)$ and $\Jmuki (s) < \Jmuki (s^{\prime})$, then $\mu^{k,i+1}(x) \subseteq U^{k,*}(x)$ and $J_{\mu^{k,i+1}}(x) > J^{k,*}(x)$.
%
%% % % % %\item $\Jmuki (u) \neq J^{k,*}(u)$ and $\Jmuki (v) \neq J^{k,*}(v)$ 
%\end{enumerate}
%
\end{lemma}

\begin{proof}
We will first show that the function $H(x,\cdot,\Jmuki)$ in (\ref{def:Hfun}) obeys a strict inequality when evaluated  at elements of the sets $U^{k,*}(x)$ and  $U^{k,\prime}(x)$. 
At the beginning of the $(i+1)$th policy improvement step, the new policy $\mu^{k,i+1}$ is computed as follows. For all $x \in B^{k,i}$
\begin{align*}
\mu^{k,i+1} (x) &\in \arg \min_{u \in U^{k}(x)} H(x,u,\Jmuki) \\
&=  \arg\min\limits_{ u \in U^{k}(x)} \Bigl\{ g(x,u) +  \Jmuki (f(x,u))\Bigr\}. 
\end{align*}
Let $u^{*} \in U^{k,*}(x)$ and $u^{\prime} \in U^{k,\prime}(x)$. 
We then have the following:
\begin{align*}
%\min H(x,u,\Jmuki) &= \min g(x,u) +  \Jmuki( f(x,u)) \\
H(x,u^{*},\Jmuki) &= g(x,u^{*}) +  \Jmuki( f(x,u^{*})) \\
&= g(x,u^{*}) +  \Jmuki(x^{*}) = g(x,u^{*}) +  J^{k,*}(x^{*}) \\
&<  g(x,u^{\prime}) +  J^{k,*}( x^{\prime}) \\
&= g(x,u^{\prime}) +  J^{k,*}( f(x,u^{\prime})) \\
&\leq g(x,u^{\prime}) +  \Jmuki( f(x,u^{\prime}))\\
&= H(x,u^{\prime},\Jmuki).
\end{align*}
This implies that
$$
H(x,u^{*},\Jmuki) < H(x,u^{\prime},\Jmuki) \quad \forall u^{*} \in U^{*}(x), u^{\prime} \in U^{\prime}(x) .
$$
Hence, it follows that
$$
\mu^{k,i+1} (x) \in \argmin_{u \in U^{k}(x)} H(x,u,\Jmuki) = \argmin_{u \in U^{k,*}(x)} H(x,u,\Jmuki)
$$
and thus $\mu^{k,i+1}(x) \subseteq U^{k,*}(x)$ for all $x \in B^{k,i}$.
Let $x \in B^{k,i}$. 
The cost function $J_{\mu^{k,i+1}}$ for $x^{*} \in S^{k,*}(x)$ is computed during the policy evaluation step for the new policy $\mu^{k,i+1}$, as follows
$$
J^{k,*}(x^{*}) \leq J_{\mu^{k,i+1}}(x^{*}) \leq \Jmuki(x^{*}) = J^{k,*}(x^{*}) \quad \Rightarrow \quad  J_{\mu^{k,i+1}}(x^{*}) = J^{k,*}(x^{*}) \quad \forall x^{*} \in S^{k,*}(x)
$$
allowing us to write $J_{\mu^{k,i+1}} (x)$ as follows:
$$
J_{\mu^{k,i+1}} (x) = g(x,u^{*}) +  J_{\mu^{k,i+1}}( f(x,u^{*})) = g(x,u^{*}) +  J_{\mu^{k,i+1}}(x^{*}) = g(x,u^{*}) +  J^{k,*}(x^{*}) = J^{k,*}(x),
$$
which implies that $J_{\mu^{k,i+1}} (x) = J^{k,*}(x)$. \QED
%
%%\item 
%\begin{comment}
%Despite having $\Jmuki (s) > J^{k,*}(s)$, one can still show the following property by using the steps given in the proof of (i)
%$$
%H(x,u^{*},\Jmuki) < H(x,u^{\prime},\Jmuki) \quad \forall u^{*} \in U^{*}(x), u^{\prime} \in U^{\prime}(x) .
%$$
%as long as $\Jmuki (s) < \Jmuki (s^{\prime})$ holds for all $s \in S^{k,*}$ and $s^{\prime} \in S^{k,\prime}$. Then, it follows that $\mu^{k,i+1}(x) \subseteq U^{k,*}(x)$.
%
%By using this information, we can compute $J_{\mu^{k,i+1}} (x)$ as follows:
%$$
%J_{\mu^{k,i+1}} (x) = g(x,u^{*}) +  J_{\mu^{k,i+1}}( f(x,u^{*})) = g(x,u^{*}) +  J_{\mu^{k,i+1}}( s) > g(x,u^{*}) +  J^{k,*}( s) = J^{k,*}(x)
%$$
%which implies that $J_{\mu^{k,i+1}} (x) = J^{k,*}(x)$.
%%\item $H(x,u,\Jmuki) = g(x,u) +  \Jmuki( f(x,u)), \quad x \in X, u \in U(x).$
%%$U^{\prime}(x) = U(x) \setminus U^{*}(x)$
%%\end{enumerate}
%\end{comment}
\end{proof}

%%%%$T_{\mu^{k,i+1}} \Jmuki = T \Jmuki$
%%%%
%%%%1) If $\Jmuki$ is equal to the optimal cost-to-go value $J^{k,*}$, then the next policy $\mu^{k,i+1}$ will be equal to the optimal policy $\mu^{k,*}$ and its associated cost-to-go value $J_{\mu^{k,i+1}}$ will be equal to the optimal cost-to-go value $J^{k,*}$. Having the optimal cost-to-go values $J^{k,*}$ and performing the policy improvement over its values is sufficient to yield the optimal control policy and its associated optimal cost-to-go value. However, this condition is not needed when considering a single state. 
%%%%
%%%%a) $\Jmuki$

\begin{lemma}\label{lemma:evolution_of_vertices_along_optimal_paths_under_policy_improvement}
Let the policy $\mu^{k,i}$ and its corresponding cost function $\Jmuki$, and assume that $P^{k,\ell} \subseteq O^{k,i}$. 
Then $\partial P^{k,\ell} \subseteq B^{k,i}$, which implies that $P^{k,\ell+1} \subseteq B^{k,i}$ before the beginning of the $(i+1)$th policy improvement step. 
Furthermore, $P^{k,\ell+1} \subseteq O^{k,i+1}$ after the  $(i+1)$th policy improvement step.
\end{lemma}

\begin{proof}
%By assumption, $P^{k,\ell} \subseteq O^{k,i}$.
As shown in Corollary~\ref{corollary1}, for all $x \in \partial P^{k,\ell}$, there exists $\succx \in P^{k,\ell}$ such that $\succx \in S^{k,*}(x)$.
The last inclusion which, in particular, that $\succx \in  \PrcSuccessor(\graph{G}^{k},x)$, equivalently, $x \in \PrcPredecessor(\graph{G}^{k}, \succx)$. 
Since, by assumption, $P^{k,\ell} \subseteq O^{k,i}$, we have that $\succx \in O^{k,i}$ and thus the following holds:
$$
\Jmuki (\succx) = J^{k,*} (\succx) < J^{k,*} (\xinit) \leq \Jmuki (\xinit) \quad \Rightarrow \quad \Jmuki (\succx) < \Jmuki (\xinit).
$$
Therefore, $\overline{\PrcPredecessor}(\graph{G}^{k}, \succx) \in B^{k,i}$, which implies that all vertices of $\partial P^{k,\ell}$ are selected for a Bellman update before the $(i+1)$th policy improvement step, and hence $\partial P^{k,\ell} \subseteq B^{k,i}$ and $P^{k,\ell+1} \subseteq B^{k,i}$.

From Corollary~\ref{corollary1} we have that $P^{k,\ell} \subseteq P^{k,\ell+1}$. 
Since the sequence $O^{k,i}$ is non-decreasing  (Lemma~\ref{lemma:nondecreasing_set_vertices_of_optimal_cost_value}),
%i.e., $O^{k,i} \subseteq O^{k,i+1}$, this implies 
it follows
that $P^{k,\ell} \subseteq O^{k,i} \subseteq O^{k,i+1}$. 
Therefore, in order to prove that $P^{k,\ell+1} \subseteq O^{k,i+1}$
we only need to show that  $\partial P^{k,\ell} \subseteq O^{k,i+1}$ by the end of the $(i+1)$th policy improvement.
From Lemma~\ref{lemma:conditions_for_achieving_optimal_policy_and_cost}, and since $\Jmuki (\succx) = J^{k,*}(\succx), \, \succx \in S^{k,*}(x)$, all vertices of $\partial P^{k,\ell}$ achieve their optimal policy and cost function value after the end of the policy improvement step, and thus $\mu^{k,i+1}(x) \subseteq U^{k,*}(x)$ and $J_{\mu^{k,i+1}}(x) = J^{k,*}(x)$. This implies that $\partial P^{k,\ell} \subseteq O^{k,i+1}$, thus completing the proof. \QED
\end{proof}

\begin{lemma}\label{lemma:nondecreasing_set_vertices_of_optimal_costvalue}
All vertices whose optimal cost value is less than that of $\xinit$, and which are part of an optimal path from $\xinit$ to $\Xgoal$ whose length is less than or equal to $i$, achieve their optimal cost value at the end of the $i$th policy improvement step, that is, $P^{k,i} \subseteq O^{k,i}$ for $i = 0, 1, \ldots$ when using policy $\mu^{k,i}$.
%and its corresponding cost function $\Jmuki$. 
%Let $S^{n,k}$ denote the set of vertices which are selected for policy improvement, i.e., $S^{n,k} = \{ \{x, \PrcPredecessor(\graph{G}^{n}, x) \}: J_{{\mu}^{n,k}}(x) < J_{{\mu}^{n,k}} (\xinit) \}$.  Then, we have $C^{n,k} \subseteq S^{n,k}$. This implies that all optimal paths which have less than $k$ edges are computed after $k$the policy improvement step. 
\end{lemma}

\begin{proof}
The claim $P^{k,i} \subseteq O^{k,i}$ will be shown using induction.

\begin{description}
\item[Basis $i = 0$:] First, note that $\Vgoal^{k} \neq \varnothing$. Let us now assume that $x \in \Vgoal^{k}$. 
Then $N^{k}(x) = 0$ for all $k=1,2,\ldots$, and $J^{k,*} (x) = 0 < J^{k,*} (\xinit)$. 
Therefore, $P^{k,0} = \Vgoal^{k}$. Also, for all $x \in P^{k,0}$, we have that $J_{\mu^{k,0}}(x) = J^{k,*} (x) = 0 < J^{k,*} (\xinit)$, which implies $P^{k,0} \subseteq O^{k,0}$.

\item[Basis $i = 1$:] 
The set of vertices along optimal paths whose length is less than or equal to 1 is a subset of goal vertices and their predecessors, that is, $P^{k,1} = P^{k,0} \cup \{ x \in V^{k} : \exists x^{\prime} \in V^{k} \cap \Xgoal \text{ s.t. } x \in \PrcPredecessor(\graph{G}^{k},\succx) ,\, \costValue(x,\succx) < J^{k,*} (\xinit) \}$. For all $\succx \in \Vgoal^{k}$, we have that $J_{{\mu}^{k,0}}(\succx) = 0 < J_{{\mu}^{k,0}} (\xinit)$. 
Therefore, all goal vertices and their predecessors are selected for Bellman update at the beginning of the first policy improvement step, hence $B^{k,0} = \{ \overline{\PrcPredecessor}(\graph{G}^{k},\succx) : \succx \in \Vgoal^{k} \}$, which implies that $P^{k,1} \subseteq B^{k,0}$. 
All vertices in $P^{k,1}$ will achieve their optimal cost values at the end of the first policy improvement step, that is, $J_{\mu^{k,1}} (x) = J^{k,*} (x) = \costValue(x,\succx)$, where $x \in P^{k,1}$, $\succx \in S^{k,*}(x)$ and $\succx \in \Vgoal$, which implies that $P^{k,1} \subseteq O^{k,1}$.

\item [Inductive step:]
Let us now assume that $P^{k,i} \subseteq O^{k,i}$ holds. 
We need to show that this assumption implies that $P^{k,i+1} \subseteq O^{k,i+1}$ at the end of $(i+1)$th policy improvement step. 
The proof of this statement follows directly from Lemma~\ref{lemma:evolution_of_vertices_along_optimal_paths_under_policy_improvement} by taking $\ell = i$. \QED
\end{description}
\end{proof}

\begin{theorem}[Optimality of Each Iteration] 
The optimal action and the optimal cost value for the initial vertex is achieved when the $\PrcReplan$ procedure of the \emph{PI}-\AlgRRTsharp{} algorithm terminates after a finite number of policy improvement steps.
%Let $\graph{G}^{k} = (V^{k},E^{k})$ be the graph built at $k$th iteration. In the $\PrcReplan$ procedure, the optimal control at the initial vertex $\mu^{k,*}(\xinit)$ and its corresponding optimal cost function value $J^{k,*}(\xinit)$ are computed at most $t_{\mathrm{f}}^{k,*} (\xinit)$ steps where $t_{\mathrm{f}}^{k,*} (\xinit)$ is the shortest length among all optimal paths between $\xinit$ and $\Xgoal$ encoded in the graph $\graph{G}^{k}$.
\end{theorem}

\begin{proof}
%The initial vertex $\xinit$ is in $P^{k,N}$ by definition where $N = t_{\mathrm{f}}^{k,*} (\xinit)$. It is shown that $P^{k,i} \in O^{k,i}$ when the $\PrcReplan$ terminates at the end of $i$th iteration in Lemma~[4].
We will investigate the case in which the algorithm terminates before performing $N^{k}(\xinit)$ policy improvement steps, where $k$ in the number of iterations the PI-\AlgRRTsharp{} algorithm has performed up to that point. 
Otherwise, optimality follows directly from Lemma~\ref{lemma:nondecreasing_set_vertices_of_optimal_costvalue}.
%if the $\PrcReplan$ procedure terminates after performing more than or equal to %$N^{k}(\xinit)$ policy improvement steps.

To this end, assume, on the contrary, that the $\PrcReplan$ procedure terminates at the end of the $i$th policy improvement step at the $k$th iteration of the 
PI-\AlgRRTsharp{} algorithm with a suboptimal cost function value for the initial vertex, that is, assume that $J_{\mu^{k,i-1}}(\xinit) > J^{k,*}(\xinit)$. 
Since the termination condition holds, there will be no policy update for all vertices in $B^{k,i-1}$.
That is, for all $x \in B^{k,i-1}$, we have that $\mu^{k,i}(x) = \mu^{k,i-1}(x)$. 
From Lemma~\ref{lemma:nondecreasing_set_vertices_of_optimal_costvalue} it follows that $P^{k,i-1} \subseteq O^{k,i-1}$ for all $i = 1, 2, \ldots$. 
This implies that $P^{k,i} \subseteq B^{k,i-1}$ at the beginning of the $i$th policy improvement step and $P^{k,i} \subseteq O^{k,i}$ at the end of $i$th policy improvement
step because of Lemma~\ref{lemma:evolution_of_vertices_along_optimal_paths_under_policy_improvement}. 
As a result, all vertices in $P^{k,i}$ achieve their optimal action and their optimal cost value at the end of the $i$th policy improvement step. 
Consequently, for all $x \in P^{k,i}$, we have that $\mu^{k,i}(x) = \mu^{k,*}(x) \text{ and } \Jmuki (x) = J^{k,*} (x)$.

Since for all vertices in $B^{k,i-1}$ there is no update observed between policies $\mu^{k,i}$ and $\mu^{k,i-1}$, we have that $\mu^{k,i}(x) = \mu^{k,i-1}(x) = \mu^{k,*}(x)$ for all $x \in P^{k,i} \subseteq B^{k,i-1}$.
Next, we investigate the cost function value of the vertices in $P^{k,i}$ at the beginning of $i$th policy improvement step and reach a contradiction.

We already know that vertices in $P^{k,i-1}$ have achieved their optimal cost function values. 
Since  $P^{k,i} = P^{k,i-1} \cup \partial P^{k,i-1}$, we thus only need to check the cost values for all the vertices in the boundary set $\partial P^{k,i-1}$. 
For all vertices in $\partial P^{k,i-1}$, their cost function values can be expressed as $J_{\mu^{k,i-1}} (x) = g(x, \mu^{k,i-1}(x)) +  J_{\mu^{k,i-1}}( f(x, \mu^{k,i-1}(x)))$. 
We already know that $\mu^{k,i-1}(x) = \mu^{k,*}(x)$ holds for all vertices in $\partial P^{k,i-1} \subseteq B^{k,i-1}$. 
Let us define $\mu^{k,*}(x) = u^{*} \in U^{k,*}(x)$ and $x^{*}\in S^{k,*}(x)$ such that $x^{*} = f(x,u^{*})$. 
Since $x^{*} \in P^{k,i-1}$, the optimal successor achieves its optimal cost  value, that is, $J_{\mu^{k,i-1}}(x^{*}) = J^{k,*} (x^{*})$. 
Then, for all vertices in $\partial P^{k,i-1}$, we can express their cost function value as $J_{\mu^{k,i-1}} (x) = g(x, u^{*}) +  J^{k,*}(\succx) = J^{k,*} (x)$. 
This implies that all vertices in $P^{k,i}$ already have achieved their optimal action and the cost values at the beginning of the $i$th policy improvement step, that is, $P^{k,i} \subseteq O^{k,i-1}$. 
We have thus shown that $ P^{k,i-1} \subseteq O^{k,i-1}$ implies $P^{k,i} \subseteq O^{k,i-1}$ for all $i = 1, 2, \ldots$. 
It follows that $ P^{k,\ell} \subseteq O^{k,i-1}$ for $\ell = 0, 1, \ldots$.

Next, consider the case when $\ell = N^{k}(\xinit)-1$. 
From the previous analysis this implies that $P^{k,\ell} \subseteq O^{k,i-1}$, which, in turn, 
implies that all vertices which may be intermediate vertices along optimal paths between $\xinit$ and the 
goal region achieve their optimal action and cost value at the beginning of the $i$th policy improvement step. 
Note that $\xinit$ is selected for a Bellman update at the beginning of the 
$i$th policy improvement step, since its cost function value can be written as 
$J_{\mu^{k,i-1}} (\xinit) = g(\xinit, u) +  J_{\mu^{k,i-1}} (\succx)$, 
where $u \in U^{k}(\xinit)$ and $\succx \in S^{k}(\xinit)$ such that $ u = \mu^{k,i-1}(\xinit)$ and $\succx = f(\xinit, u)$. 
This implies that $\xinit \in \PrcPredecessor(\graph{G}^{k},\succx)$ and $J_{\mu^{k,i-1}} (\succx) < J_{\mu^{k,i-1}} (\xinit)$ and therefore, 
$\xinit \in B^{k,i-1}$. 
However, since the termination condition holds, a Bellman update for $\xinit$ does not yield any update in its action during the $i$th policy improvement step, and thus $\mu^{k,i}(\xinit) = \mu^{k,i-1}(\xinit)$. 
We also know that, 
Since $S^{k,*}(\xinit) \subseteq P^{k,\ell}$ and $P^{k,N-1} \subseteq O^{k,i-1}$, it follows 
that all vertices in $S^{k,*}(\xinit)$ have achieved their optimal action and their optimal cost value at the beginning of the $i$th policy iteration.
That is, $\mu^{k,i-1}(\succx) = \mu^{k,*}(\succx)$ and $J^{k,i-1}(\succx) = J^{k,*}(\succx)$ for all $\succx \in S^{k,*}(\xinit)$. 
It follows from  Lemma~\ref{lemma:conditions_for_achieving_optimal_policy_and_cost} that $\mu^{k,i}(\xinit) \subseteq U^{k,*}(\xinit)$ and $\Jmuki(\xinit) = J^{k,*}(\xinit)$ 
at the end of $i$th policy improvement step.
This implies that $\mu^{k,i-1}(\xinit) = \mu^{k,i}(\xinit) = \mu^{k,*}(\xinit)$. 
The cost value of $\xinit$ at the beginning of $i$th policy improvement step  is given 
by
$J_{\mu^{k,i-1}} (\xinit) = g(x, \mu^{k,i-1}(\xinit)) +  J_{\mu^{k,i-1}}( f(x, \mu^{k,i-1}(\xinit)))$. 
We know that $\mu^{k,i-1}(\xinit) = \mu^{k,*}(x)$. 
Let $\mu^{k,*}(\xinit) = u^{*} \in U^{k,*}(\xinit)$ and $\succx \in S^{k,*}(\xinit)$ such that $\succx = f(x,u^{*})$. 
Since $\succx \in P^{k,\ell}$, and $\ell = N^{k}(\xinit)-1$, $\succx$ achieves the optimal cost function value, and hence $J_{\mu^{k,i-1}}(\succx) = J^{k,*} (\succx)$. 
We thus have 
$J_{\mu^{k,i-1}} (\xinit) = g(x, u^{*}) +  J^{k,*}( \succx) = J^{k,*} (\xinit)$.

We have thus shown that $J_{\mu^{k,i-1}} (\xinit) = J^{k,*} (\xinit)$ which leads to the contradiction we seek, given the initial assumption that the algorithm
terminates with a suboptimal cost value for the initial vertex. \QED
\end{proof}

The previous theorem states that when the $\PrcReplan$ procedure terminates at the beginning of the $N^{k}(\xinit)$th policy improvement step, it has already computed 
the optimal action and cost function value for $\xinit$. 
If the algorithm terminates after more than or equal to $N^{k}(\xinit)$ policy improvement steps, then optimality follows directly 
from Lemma~\ref{lemma:nondecreasing_set_vertices_of_optimal_costvalue}, since
the $\PrcReplan$ procedure is thus guaranteed to terminate after a finite number of policy improvement steps, owing to the properties of policy iteration
and the fact that the policy space is finite~\cite{Bertsekas2000}.

\begin{theorem}[Termination of $\PrcReplan$  Procedure after a Finite Number of Steps] 
Let $\graph{G}^{k} = (V^{k},E^{k})$ be the graph built at the end of $k$th iteration of the \emph{PI}-\AlgRRTsharp{} algorithm. 
Then, the $\PrcReplan$ procedure of the \emph{PI}-\AlgRRTsharp{} algorithm terminates after at most $(\overline{N}^{k}+2)$ policy improvement steps, where $\overline{N}^{k}  = \max_{ x \in V_{\mathrm{prom}}^{k} } N^{k}(x)$ and
$\overline{V}_{\mathrm{prom}}^{k} = \{ \overline{\PrcPredecessor}(\graph{G}^{k}, x) : x \in \Vprom^{k} \}$.
\end{theorem}

\begin{proof}
Let us assume, on the contrary, that the $\PrcReplan$ procedure does not terminate at the end of the $(\overline{N}^{k}+2)th$ policy improvement step at the $k$th iteration of the PI-\AlgRRTsharp{} algorithm. 
This implies that there exists a point $x \in B^{k,\overline{N}^{k}+1}$ such that its cost function value is reduced, and its policy is updated at the end of the $(\overline{N}^{k}+2)th$ policy improvement step. 
Equivalently, there exists $u \in U^{k}(x)$ that yields $J_{\mu^{ k, \overline{N}^{k}+1} }(x) > g (x, u) + J_{\mu^{ k, \overline{N}^{k}+1} }( \succx )$ where $\succx \in S^{k}(x)$, $\mu^{ k, \overline{N}^{k}+1} \neq \mu^{k,\overline{N}^{k}+2}(x) = u$ and $\succx = f(x, u)$.
%We will check if the preceding inequality makes sense by analyzing the value of $J_{\mu^{ k, \overline{N}^{k}+1} }(x)$.
By definition, we have $P^{k,\overline{N}^{k}} = \Vprom^{k}$, which implies that $\Vprom^{k}\subseteq O^{k,\overline{N}^{k}} \subseteq O^{k,\overline{N}^{k}+1}$ due to Lemma~\ref{lemma:nondecreasing_set_vertices_of_optimal_costvalue} and Lemma~\ref{lemma:nondecreasing_set_vertices_of_optimal_cost_value}. 
For all $x \in \Vprom^{k}$, we have $J_{\mu^{k,\overline{N}^{k}+1}}(x) = J^{k,*}(x) < J^{k,*}(\xinit) \leq J_{\mu^{k,\overline{N}^{k}+1}}(\xinit)$, which implies that $\overline{\PrcPredecessor}(\graph{G},x) \in B^{k, \overline{N}^{k}+1}$. 
Therefore, $\overline{V}_{\mathrm{prom}}^{k} \subseteq B^{k,\overline{N}^{k}+1}$. 
Since $\Vprom^{k} \subseteq O^{k,\overline{N}^{k}}$ and $\overline{V}^{k}_{\mathrm{prom}}\subseteq B^{k,\overline{N}^{k}+1}$, it can also be shown, similarly to 
Lemma~\ref{lemma:evolution_of_vertices_along_optimal_paths_under_policy_improvement}, that all vertices of $\overline{V}^{k}_{\mathrm{prom}}$ achieve their 
optimal cost values and their optimal policies after the $(\overline{N}^{k}+1)th$ policy improvement step.
As a result, $ J_{\mu^{k,\overline{N}^{k}+1}} (x) = J^{k,*}(x), ~ \mu^{k,\overline{N}^{k}+1}(x) \subseteq U^{k,*}(x)$ for all $x \in \overline{V}_{\mathrm{prom}}^{k}$.
  
Next, note that for the successor vertex of $\xinit$ along the optimal path between $\xinit$ and the goal region we have that $N^{k}(\succx) = N^{k}(\xinit) - 1 \leq \overline{N}^{k}$ since $\succx \in \Vprom^{k}$. This implies that $N^{k}(\xinit) \leq \overline{N}^{k} + 1$, and therefore, from Lemma~\ref{lemma:nondecreasing_set_vertices_of_optimal_costvalue}, we have that $J_{\mu^{k,\overline{N}^{k}+1}} (\xinit) = J^{k,*}(\xinit)$. 
Recall now that, for all $x \in V^{k}$ with $\overline{\PrcPredecessor}(\graph{G}^{k}, x) \in B^{k,\overline{N}^{k}+1} $, we have that $J_{{\mu}^{k,\overline{N}^{k}+1}}(x) < J_{{\mu}^{k,\overline{N}^{k}+1}} (\xinit)$. 
Since $J_{\mu^{k,\overline{N}^{k}+1}} (\xinit) = J^{k,*}(\xinit)$, the following expression holds:
$$
J^{k,*}(x) \leq J_{{\mu}^{k,\overline{N}^{k}+1}}(x) < J_{{\mu}^{k,\overline{N}^{k}+1}} (\xinit) = J^{k,*}(\xinit) \quad \Rightarrow \quad
J^{k,*}(x) < J^{k,*}(\xinit)
$$
Therefore, $x \in \Vprom^{k}$ and $\overline{\PrcPredecessor}(\graph{G}^{k}, x) \in \overline{V}^{k}_{\mathrm{prom}}$ which implies that  $B^{k,\overline{N}^{k}+1} \subseteq \overline{V}^{k}_{\mathrm{prom}}$. 
From the two preceding results, it follows that $B^{k,\overline{N}^{k}+1} = \overline{V}^{k}_{\mathrm{prom}}$. 

Let $x \in B^{\overline{N}^{k}+1} = \overline{V}^{k}_{\mathrm{prom}}$, whose policy is updated 
during the $(\overline{N}^{k}+2)th$ policy improvement step.
We therefore have that $J^{k,*}(x) = J_{\mu^{ k, \overline{N}^{k}+1} }(x) > g (x, u) + J_{\mu^{ k, \overline{N}^{k}+1} }( \succx ) = g (x, u) +  J^{k,*}( \succx )$. 
This yields $J^{k,*}(x) > g (x, u) +  J^{k,*}( \succx )$, which contradicts (\ref{eqn:deterministic_system:bellmans_equation}), thus completing the proof. \QED
\end{proof}

\begin{theorem}[Asymptotic Optimality of PI-RRT$^\#$]
Let $\graph{G}^{k} = (V^{k},E^{k})$ be the graph built at the end of the $k$th iteration of the PI-\AlgRRTsharp{} algorithm and let
$N^{k}$ is maximum number of policy improvement steps performed at the $k$ iteration.
As $k \rightarrow \infty$, the policy $\mu^{k,N^{k}}(\xinit)$ and its corresponding cost function $J_{\mu^{k,N^{k}}}(\xinit)$,  
converge to the optimal policy $\mu^{*}(\xinit)$ and corresponding optimal cost function $J_{\mu^{*}}(\xinit)$ with probability one.
\end{theorem}

\begin{proof}
The graph $\graph{G}^{k} = (V^{k},E^{k})$ is constructed by the \AlgRRG{} algorithm at the beginning of $k$th iteration. In the PI-\AlgRRTsharp{} algorithm, the optimal cost function value of $\xinit$ with respect to $\graph{G}^{k}$ is computed during the $\PrcReplan$ procedure at the end of $k$th iteration, that is, 
$J_{\mu^{k,N^{k}}} (\xinit) = J^{k,*} (\xinit)$ and $\mu^{k,N^{k}}(\xinit) = \mu^{k,*}(\xinit)$. 
Since the \AlgRRG{} algorithm is asymptotically optimal with probability one, $\graph{G}^{k}$ will encode, almost surely, 
the optimal  path between $\xinit$ and goal region as $k \rightarrow \infty$.
This implies that $J_{\mu^{k,N^{k}}} (\xinit) = J^{k,*} (\xinit) \to J^{*} (\xinit) $ and $\mu^{k,N^{k}} (\xinit) = \mu^{k,*} (\xinit) \to \mu^{*} (\xinit)$ with probability one. \QED
%Therefore, the PI-\AlgRRTsharp{} algorithm is also asymptotically optimal with probability one.
\end{proof}

\newpage

\section{Numerical Simulations}

We implemented both the baseline \AlgRRTsharp{} and PI-\AlgRRTsharp{} algorithms in MATLAB and performed Monte Carlo simulations on shortest path planning problems in two different 2D environments, namely, sparse and highly cluttered environments. 
The goal was to find the shortest path that minimizes the Euclidean distance from an initial point to a goal point. 
The initial and goal points are shown in yellow and dark blue squares in the figures below, respectively. 
The obstacles are shown in red and the best path computed during each iteration is shown in yellow.

The results were averaged over 100 trials and each trial was run for 10,000 iterations. 
No vertex rejection rule is applied during the extension procedure. 
We then computed the total time required to complete a trial and measured the time spent on the non-planning (sampling, extension, etc.) and the planning-related procedures of the algorithms, separately.  
The growth of the tree in each case is shown in Figure~\ref{figure:sim_d2_pt2_rrtsharp_v0_pi_iterations}. 
At each iteration, a subset of promising vertices is determined during the policy evaluation step and policy improvement is performed only for these vertices. 
The promising vertices are shown in magenta in Figure~\ref{figure:sim_d2_pt2_rrtsharp_v0_pi_iterations}.

\begin{figure*}[!ht]

\centering
	\mbox{
	% previous size 257
    \renewcommand{\thesubfigure}{(a)} \subfigure[]{\scalebox{0.29}{\includegraphics[trim = 4.0cm 6.937cm 3.587cm 7.0cm, clip =
          true]{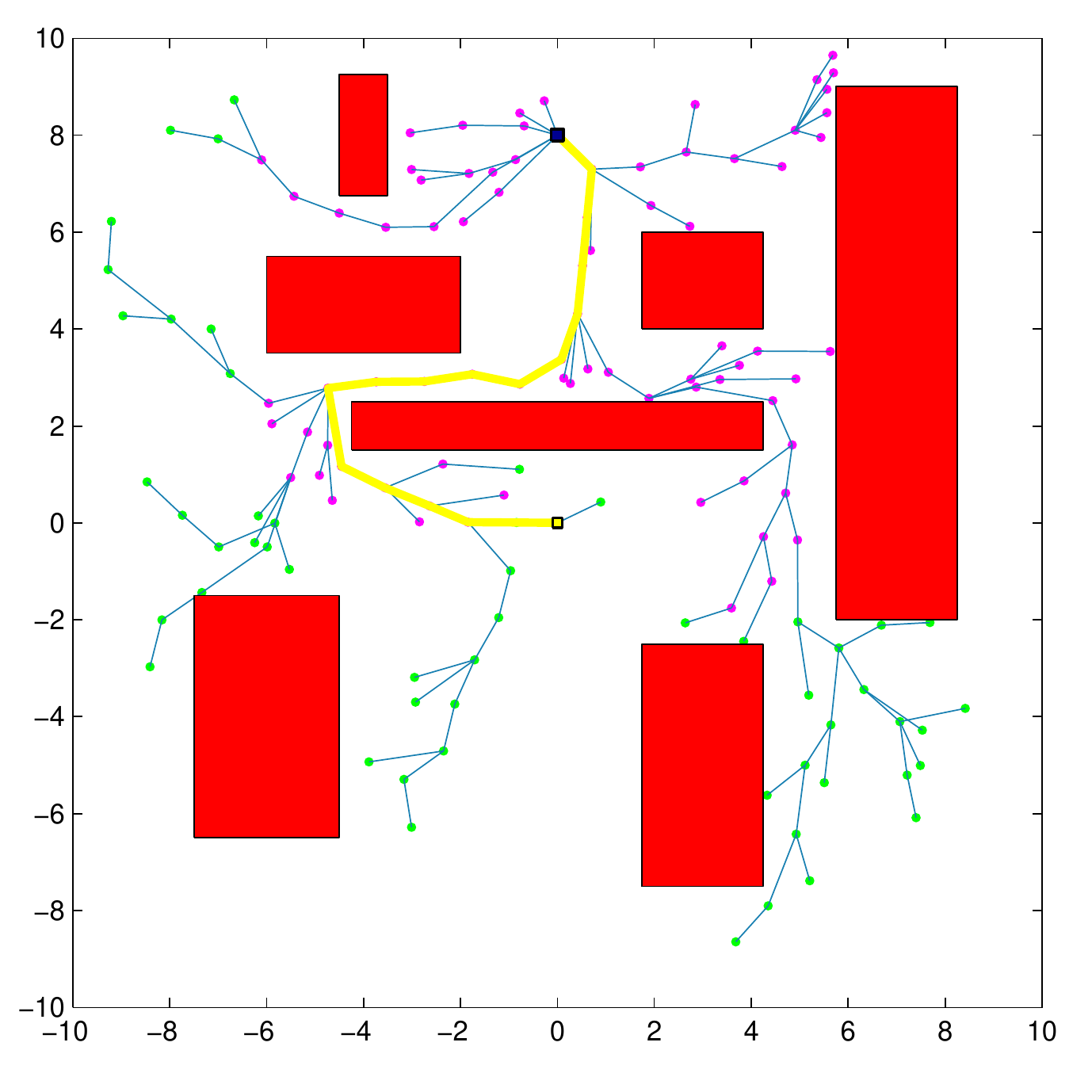}} \label{figure:pt2_rrtsharp_v0_pi_it200r}}
    \renewcommand{\thesubfigure}{(b)} \subfigure[]{\scalebox{0.29}{\includegraphics[trim = 4.0cm 6.937cm 3.587cm 7.0cm, clip =
          true]{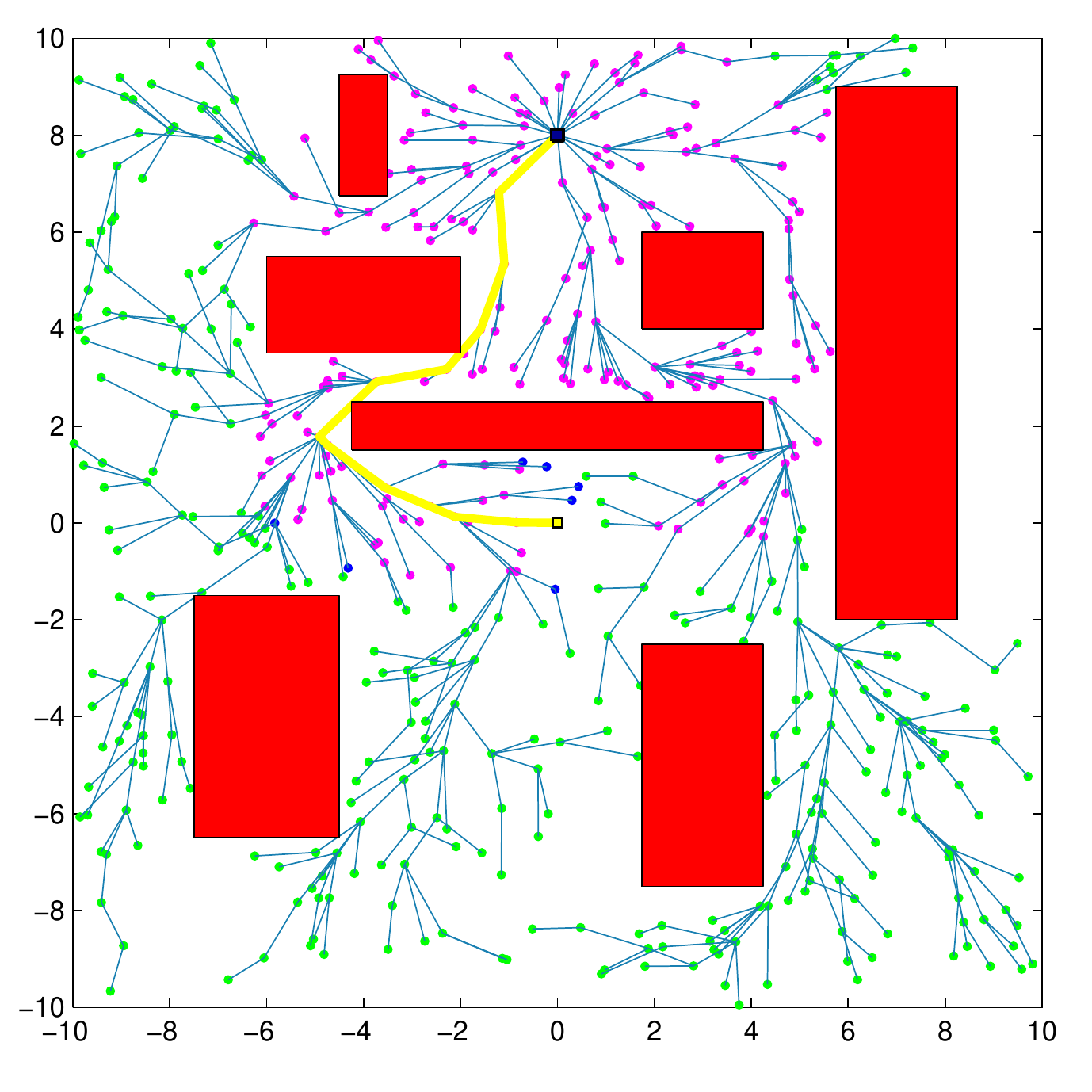}} \label{figure:pt2_rrtsharp_v0_pi_it600r}}
    \renewcommand{\thesubfigure}{(c)} \subfigure[]{\scalebox{0.29}{\includegraphics[trim = 4.0cm 6.937cm 3.587cm 7.0cm, clip =
          true]{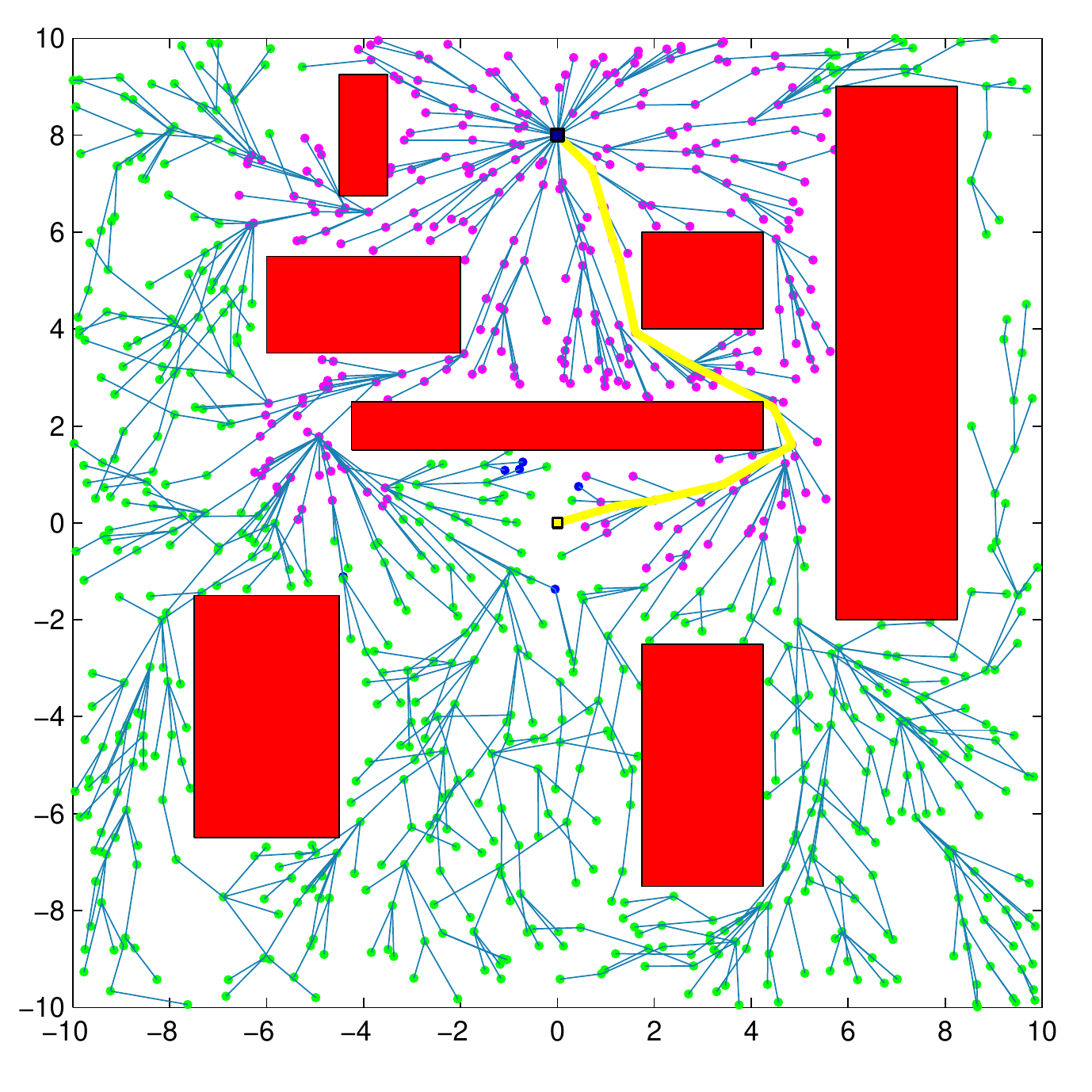}}\label{figure:pt2_rrtsharp_v0_pi_it1000r}}
    \renewcommand{\thesubfigure}{(d)} \subfigure[]{\scalebox{0.29}{\includegraphics[trim = 4.0cm 6.937cm 3.587cm 7.0cm, clip =
          true]{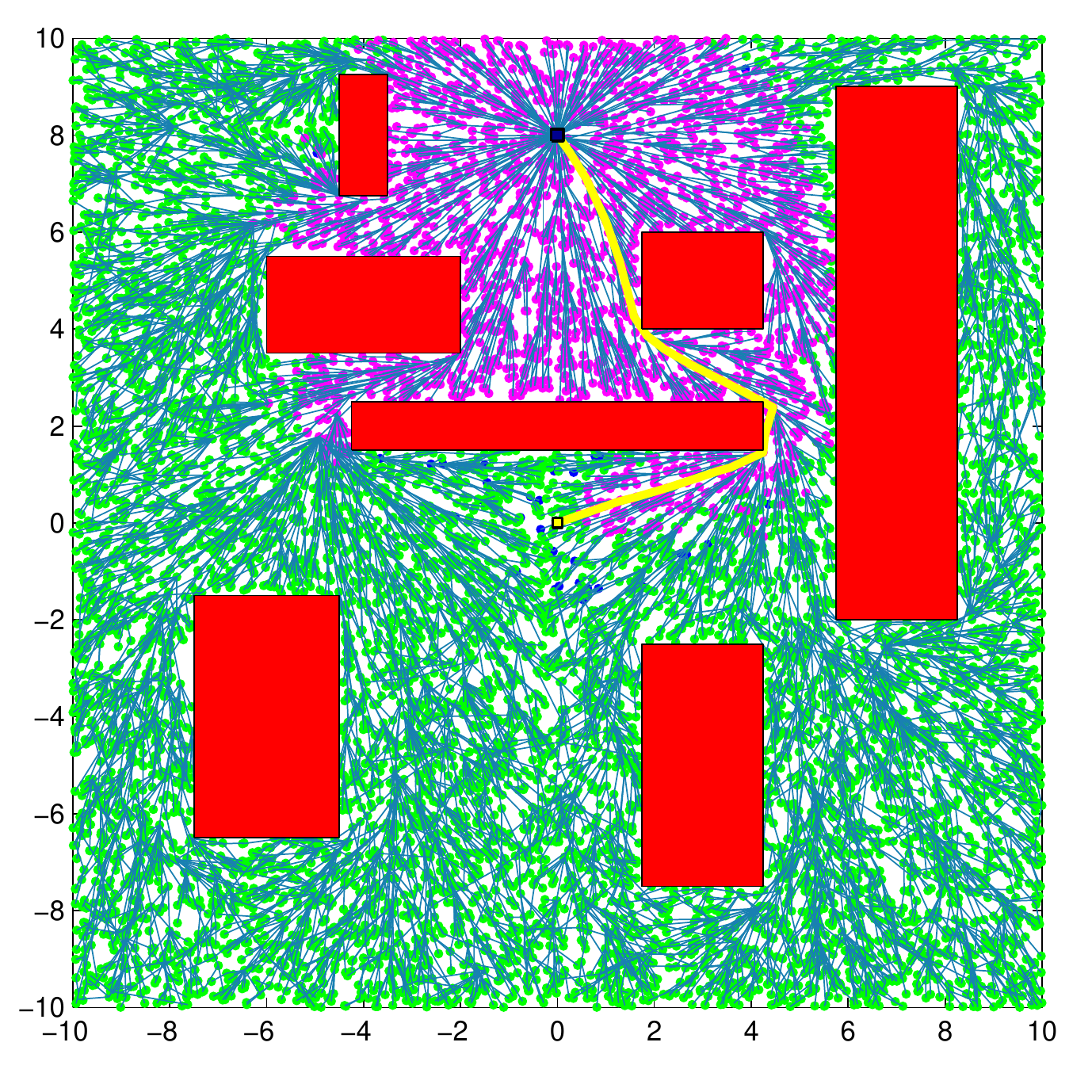}}\label{figure:pt2_rrtsharp_v0_pi_it10000r}}
	 }\vspace*{-4.5mm}
	\mbox{
    \renewcommand{\thesubfigure}{(e)} \subfigure[]{\scalebox{0.29}{\includegraphics[trim = 4.0cm 6.937cm 3.587cm 7.0cm, clip =
          true]{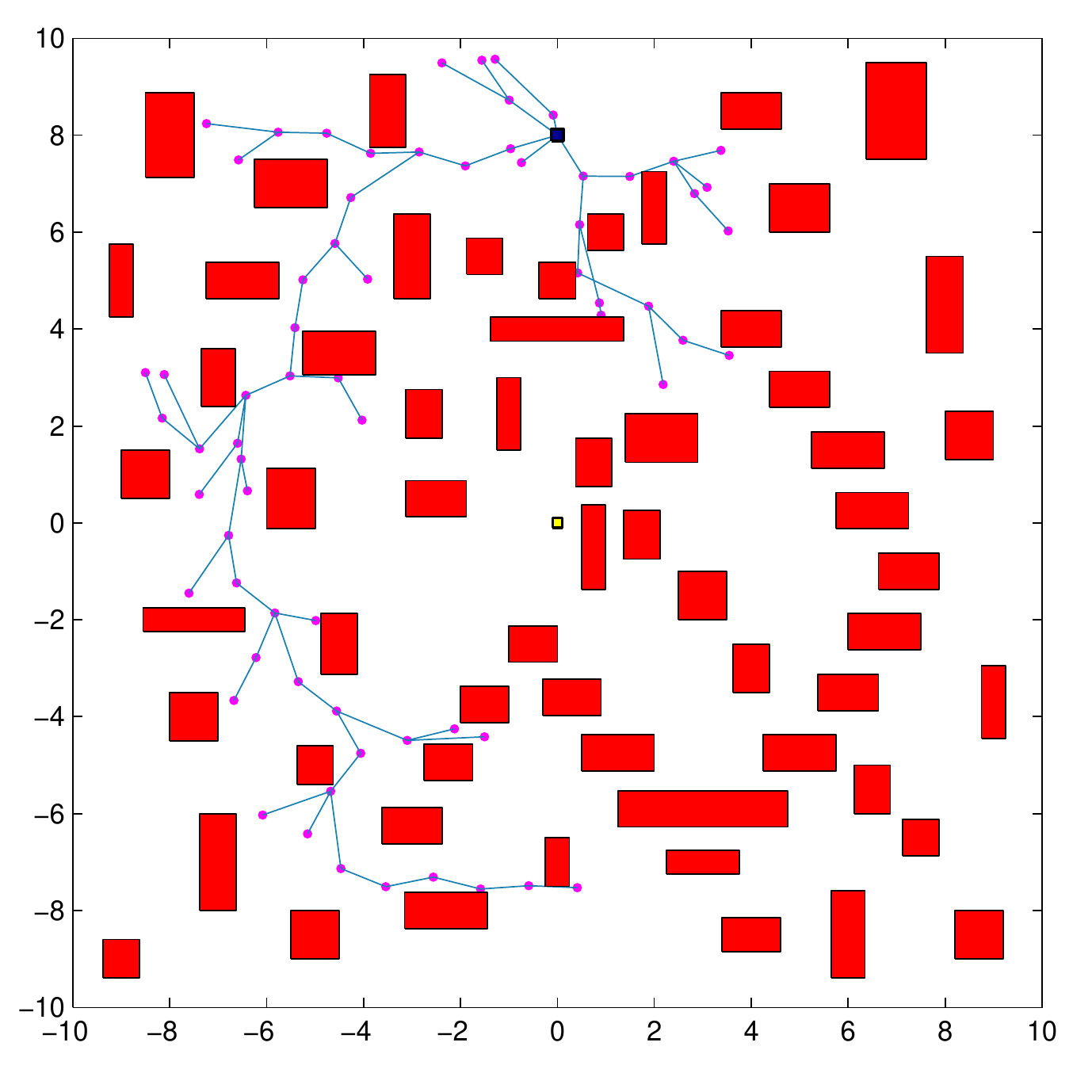}} \label{figure:pt3_rrtsharp_v0_pi_it200r}}
    \renewcommand{\thesubfigure}{(f)} \subfigure[]{\scalebox{0.29}{\includegraphics[trim = 4.0cm 6.937cm 3.587cm 7.0cm, clip =
          true]{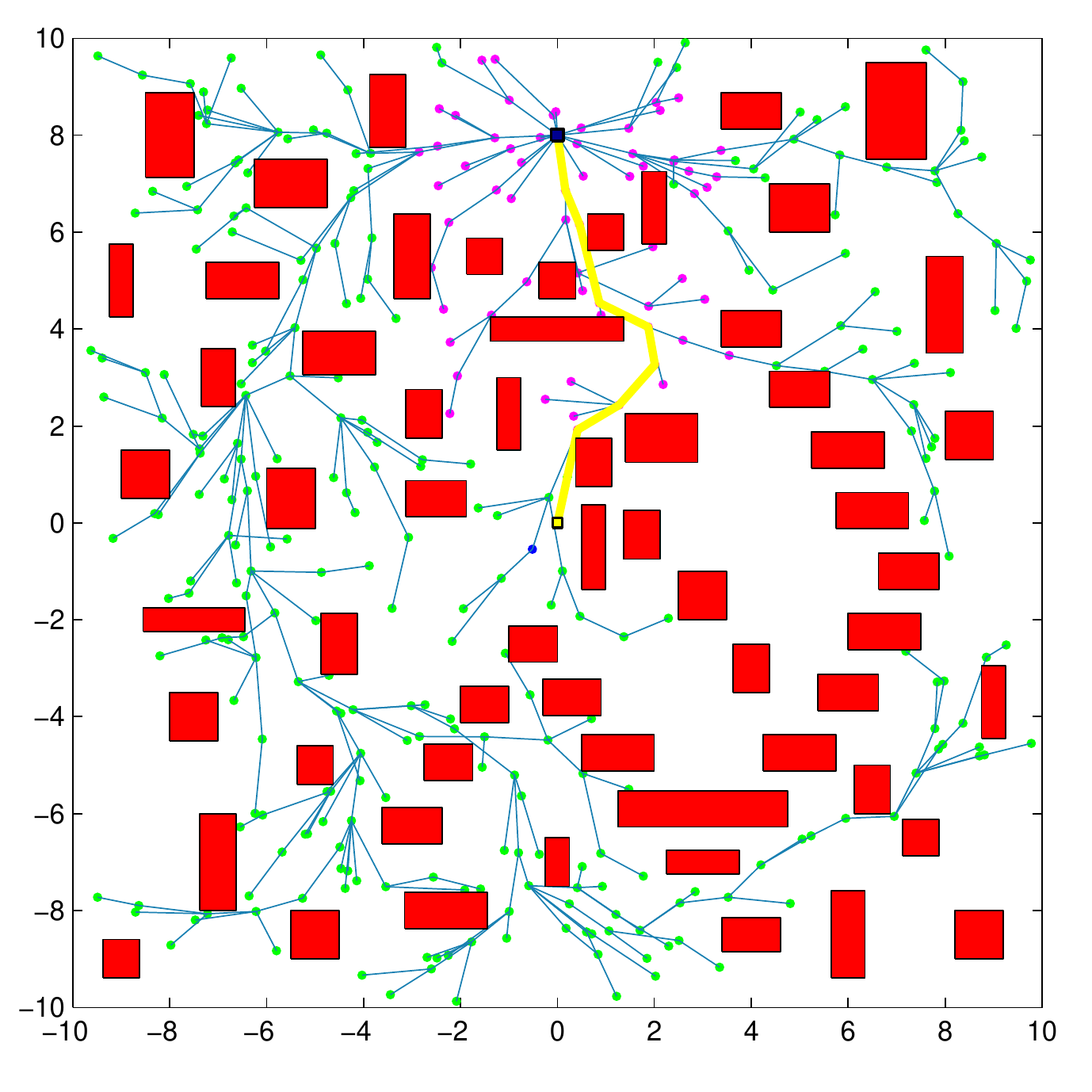}} \label{figure:pt3_rrtsharp_v0_pi_it600r}}
    \renewcommand{\thesubfigure}{(g)} \subfigure[]{\scalebox{0.29}{\includegraphics[trim = 4.0cm 6.937cm 3.587cm 7.0cm, clip =
          true]{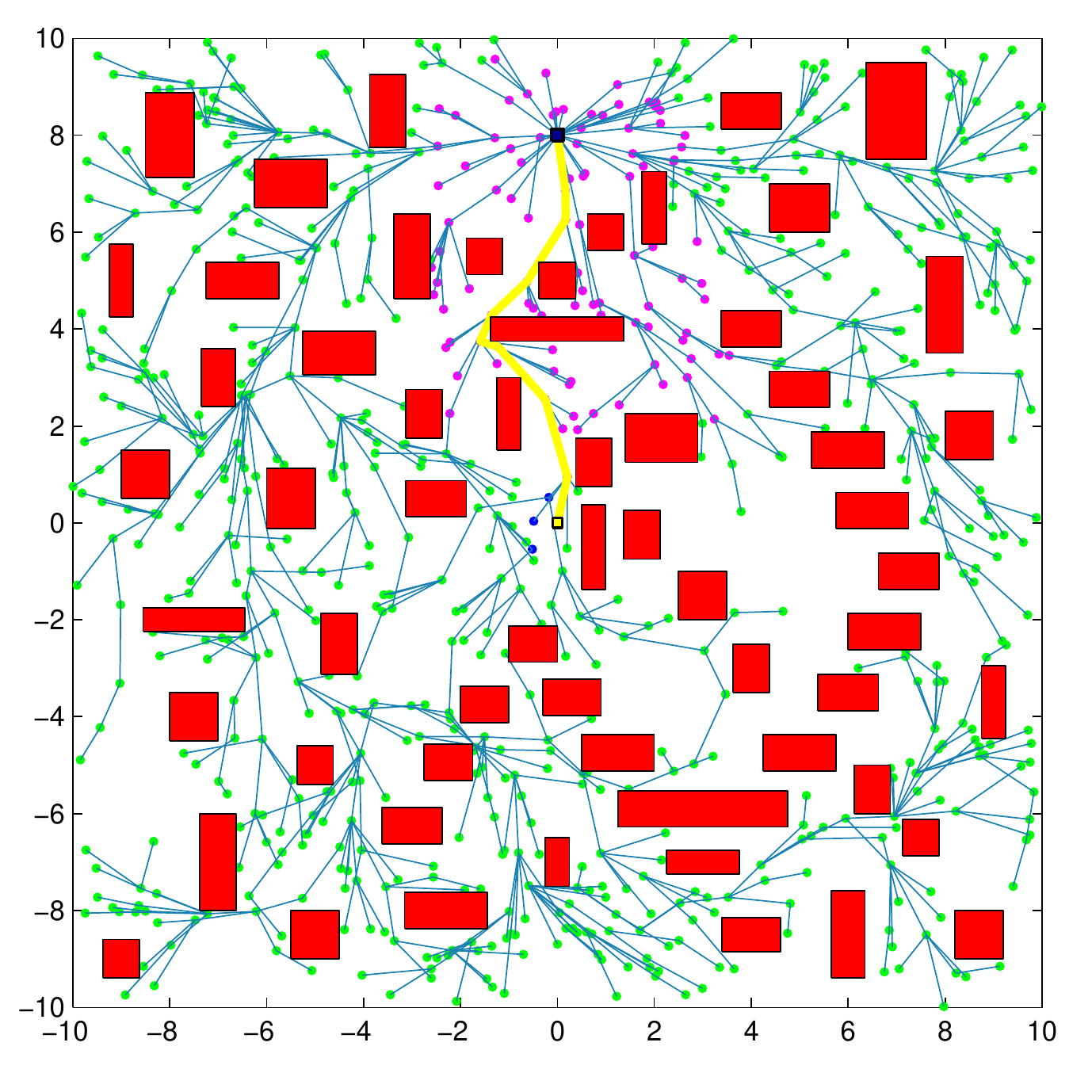}} \label{figure:pt3_rrtsharp_v0_pi_it1000r}}
    \renewcommand{\thesubfigure}{(h)} \subfigure[]{\scalebox{0.29}{\includegraphics[trim = 4.0cm 6.937cm 3.587cm 7.0cm, clip =
          true]{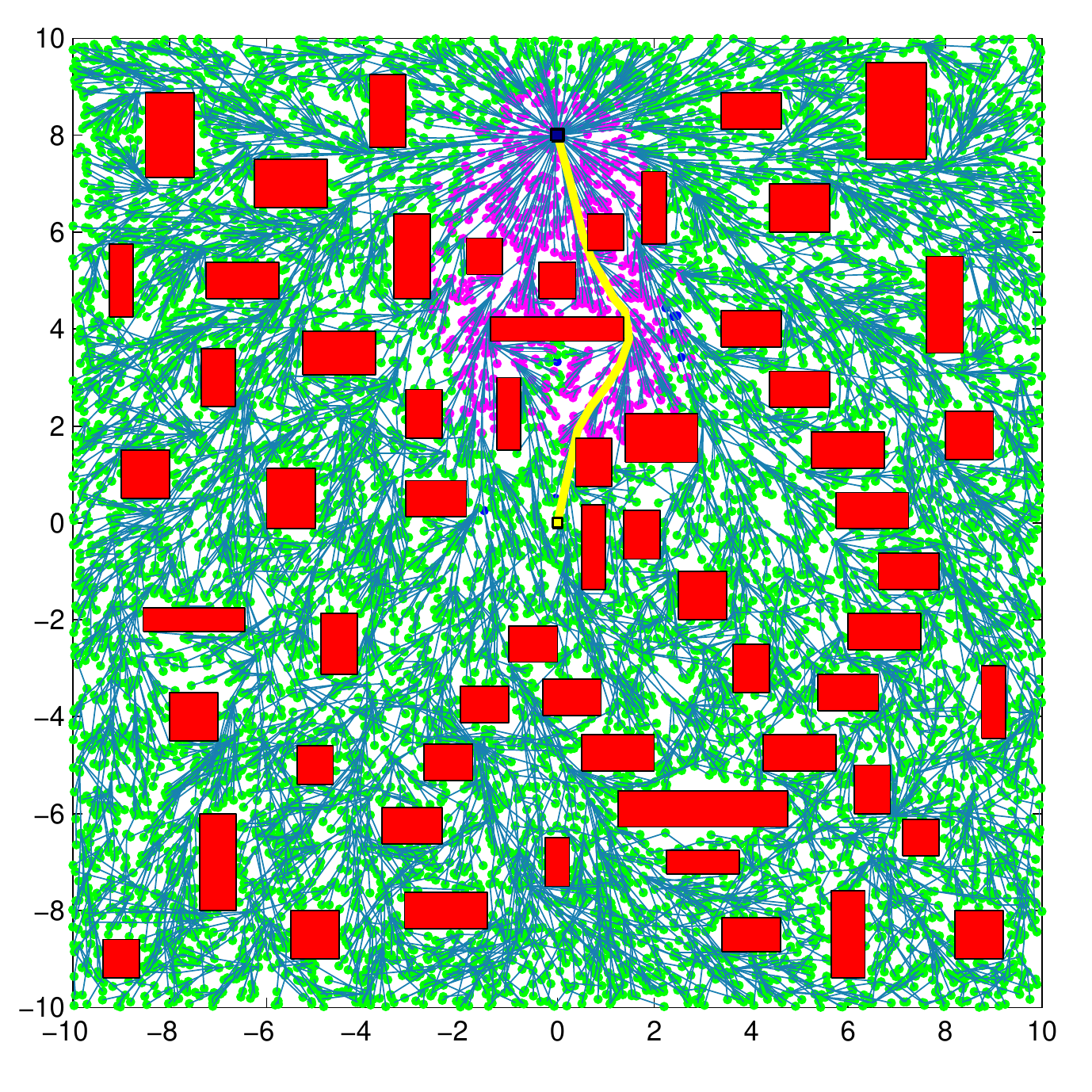}} \label{figure:pt3_rrtsharp_v0_pi_it10000r}}
	 }\vspace*{-4mm}
%	\mbox{
%    \renewcommand{\thesubfigure}{(c)} \subfigure[]{\scalebox{0.257}{\includegraphics[trim = 4.0cm 6.937cm 3.587cm 7.0cm, clip =
%          true]{fig_pt2_rrtsharp_v0_it10000r.pdf}} \label{figure:pt2_rrtsharp_v0_it24999}}
%    \renewcommand{\thesubfigure}{(f)} \subfigure[]{\scalebox{0.257}{\includegraphics[trim = 4.0cm 6.937cm 3.587cm 7.0cm, clip =
%          true]{fig_pt2_rrtsharp_v1_it10000r.pdf}} \label{figure:pt2_rrtsharp_v1_it24999}}
%    \renewcommand{\thesubfigure}{(i)} \subfigure[]{\scalebox{0.257}{\includegraphics[trim = 4.0cm 6.937cm 3.587cm 7.0cm, clip =
%          true]{fig_pt2_rrtsharp_v2_it10000r.pdf}} \label{figure:pt2_rrtsharp_v2_it24999}}
%    \renewcommand{\thesubfigure}{(l)} \subfigure[]{\scalebox{0.257}{\includegraphics[trim = 4.0cm 6.937cm 3.587cm 7.0cm, clip =
%          true]{fig_pt2_rrtsharp_v3_it10000r.pdf}} \label{figure:pt2_rrtsharp_v3_it24999}}
% }\vspace*{-4mm}
    \caption{The evolution of the tree computed by PI-\AlgRRTsharp{} algorithm is shown in  \subref{figure:pt2_rrtsharp_v0_pi_it200r}-\subref{figure:pt2_rrtsharp_v0_pi_it10000r} for the problem with less cluttered environment, and \subref{figure:pt3_rrtsharp_v0_pi_it200r}-\subref{figure:pt3_rrtsharp_v0_pi_it10000r} for the problem with cluttered environment. 
    The configuration of the trees \subref{figure:pt2_rrtsharp_v0_pi_it200r}, \subref{figure:pt3_rrtsharp_v0_pi_it200r} is at 200 iterations, \subref{figure:pt2_rrtsharp_v0_pi_it600r}, \subref{figure:pt3_rrtsharp_v0_pi_it600r} is at 600 iterations, \subref{figure:pt2_rrtsharp_v0_pi_it1000r}, \subref{figure:pt3_rrtsharp_v0_pi_it1000r} is at 1,000 iterations,
     and \subref{figure:pt2_rrtsharp_v0_pi_it10000r}, \subref{figure:pt3_rrtsharp_v0_pi_it10000r} is at 10,000 iterations.
    } \label{figure:sim_d2_pt2_rrtsharp_v0_pi_iterations}

\end{figure*}

For the first problem, the average time spent for non-planning related procedures in the \AlgRRTsharp{} and PI-\AlgRRTsharp{} algorithms are shown in blue and red colors, respectively, in Figure~\ref{figures:non_planning_time_pt2}. 
As seen from these figures, PI-\AlgRRTsharp{} is slightly faster than the \AlgRRTsharp{} algorithm, especially when adding a new vertex to the graph. 
Since there is no priority queue in the PI-\AlgRRTsharp{} algorithm, it is much cheaper to include a new vertex and there is no need for vertex ordering.

{
\setlength{\intextsep}{4pt}
\begin{figure}[htp]
\centering
% l b r t %4.0cm 3.0cm 4.0cm 3.0cm
\scalebox{0.45}{\includegraphics[trim = 1.8cm 6.9cm 2.2cm 7.1cm, clip = true]{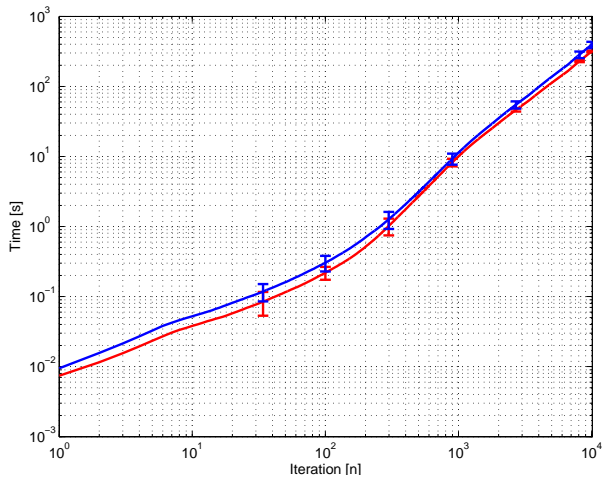}}
\caption{The time required for non-planning procedures to complete a certain number of iterations for the first problem set. The time curve for \AlgRRTsharp{} and PI-\AlgRRTsharp{} are shown in blue and red, respectively. Vertical bars denote standard deviation averaged over 100 trials.} \label{figures:non_planning_time_pt2}
%\caption{The change in the cost of the best solutions computed by the \AlgRRTstar{}, \AlgRRTsharp{} and variant algorithms. Vertical bars and filled circles denote standard deviation and completion time averaged over 100 trials, respectively.} \label{figure:time_cost_mean_d6_all}
\end{figure}
}

For the first problem, the average time spent for planning related procedures for the \AlgRRTsharp{} and the PI-\AlgRRTsharp{} algorithms are shown in blue and red colors, respectively, in Figure~\ref{figures:planning_time_pt2}. 
As seen ion those figures, the relation between time and iteration is linear when the number of iterations becomes large in the log-log scale plot, which implies a polynomial relationship, i.e., $t(n) = c n^{\alpha}$. 
One can find these parameters by using a least-square minimization based on the measured data for iterations between 100 to 10,000. 
These parameters can be computed as $c_{0} = 6.4322 \times 10^{-5}, \alpha_{0} = 1.2925$ for \AlgRRTsharp{} and $c_{\mathrm{pi}} = 1.0672 \times 10^{-6}, \alpha_{\mathrm{pi}} = 2.214$ for the PI-\AlgRRTsharp{}. 
The fitted time-iteration lines (dashed) for the \AlgRRTsharp{} and the PI-\AlgRRTsharp{} are shown in magenta and green colors, respectively. 
In our implementation, we uses one processor to perform policy improvement due to simplicity. 
However, as mentioned earlier, the policy improvement step can be done in parallel. 
One can divide the set of promising vertices into disjoint sets and assign each of them to a different processor.

{
\setlength{\intextsep}{4pt}
\begin{figure}[htp]
\centering
% l b r t %4.0cm 3.0cm 4.0cm 3.0cm
\scalebox{0.45}{\includegraphics[trim = 1.8cm 6.9cm 2.2cm 7.1cm, clip = true]{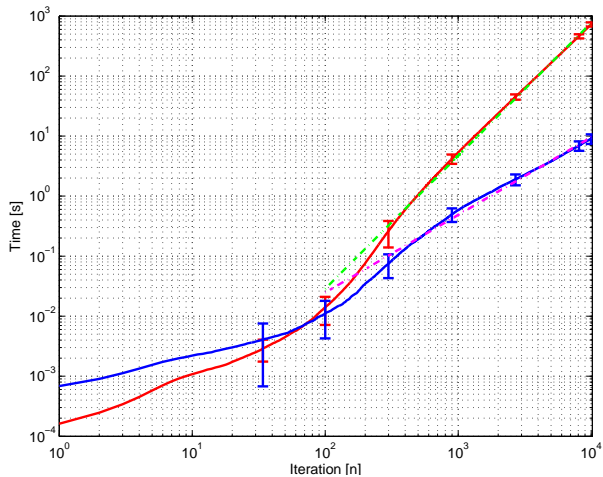}}
\caption{The time required for planning procedures to complete a certain number of iterations for the first problem set. The time curve for \AlgRRTsharp{} and PI-\AlgRRTsharp{} are shown in blue and red, respectively. Vertical bars denote standard deviation averaged over 100 trials.} \label{figures:planning_time_pt2}
%\caption{The change in the cost of the best solutions computed by the \AlgRRTstar{}, \AlgRRTsharp{} and variant algorithms. Vertical bars and filled circles denote standard deviation and completion time averaged over 100 trials, respectively.} \label{figure:time_cost_mean_d6_all}
\end{figure}
}

Let $n_{\mathrm{p}}$ denote the number of processors, and let $N$ denote the computational load per processor, i.e, $N = n/n_{\mathrm{p}}$, where $n$ is the iteration number which can be considered as an upper bound on the number of promising vertices. 
Simple calculation shows that the load per processor needs to satisfy the following relationship for the PI-\AlgRRTsharp{} algorithm in order to outperform the 
baseline \AlgRRTsharp{} algorithm for faster planning.
$$
N = \frac{n}{n_{\mathrm{p}}} > \frac{c_{0}}{c_{\mathrm{pi}}} n^{1 + \alpha_{0} - \alpha_{\mathrm{pi}}}.
$$
For the first problem, based on these empirical data, the load per processor versus iteration limit is $N > 60.27 n^{0.078562}$. 
This line is plotted in Figure~\ref{figure:load_curve_pt2}. 
For example, the load per processor needs to be smaller than 124.27, so that each processor should not be assigned more than 124.27 vertices for policy improvement. 
This implies that the number of processors needs to be greater than 80.47 during the 10,000$th$ iteration.

{
\setlength{\intextsep}{4pt}
\begin{figure}[htp]
\centering
% l b r t %4.0cm 3.0cm 4.0cm 3.0cm
\scalebox{0.45}{\includegraphics[trim = 1.8cm 6.9cm 2.2cm 7.1cm, clip = true]{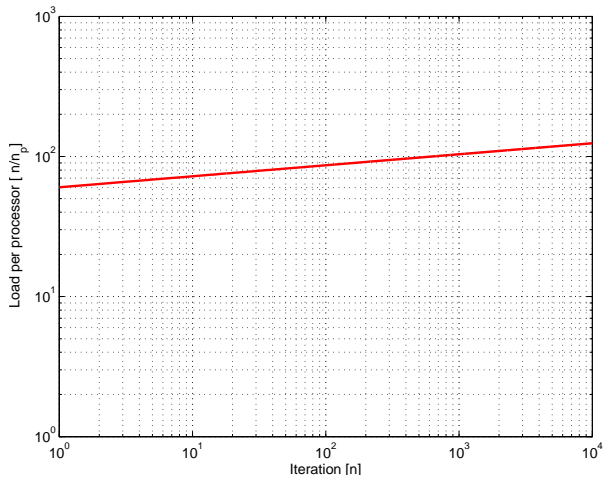}}
\caption{The number of vertices assigned per processors needs to be smaller than the load per processor line in order to outperform \AlgRRTsharp{} algorithm.} \label{figure:load_curve_pt2}

\end{figure}
}

From the previous simple analysis it follows that the PI-\AlgRRTsharp{} algorithm can be a better choice than \AlgRRTsharp{} algorithm for planning problems in high-dimensional search spaces. 
In high-dimensional search spaces, one needs to run planning algorithms for a large number of iterations in order to explore the search space densely and see a significant improvement in the computed solutions. 
This requirement induces a bottleneck on the \AlgRRTsharp{} algorithm and all similar VI-based algorithms since the re-planning procedure is performed sequentially and requires ordering of vertices. 
Therefore, this operation may take a long time, as the number of vertices increases significantly. 
On the other hand, the PI-\AlgRRTsharp{} algorithm does not require any ordering of the vertices, and one can keep re-planning tractable by employing more processors (e.g., spawning more threads) as needed, in order to meet the desired load per processor requirement. 
Given the current advancement in parallel computing technologies, such as GPUs, a well-designed parallel implementation of the PI-\AlgRRTsharp{} may yield significant real-time execution performance improvement for some problems that are known to be very challenging to handle with existing VI-based probabilistic algorithms. 

The same analysis was carried out for the second problem and the results are shown in Figure~\ref{figure:non_planning_time_pt3}, \ref{figure:planning_time_pt3} and \ref{figure:load_curve_pt3}.

{
\setlength{\intextsep}{4pt}
\begin{figure}[htp]
\centering
% l b r t %4.0cm 3.0cm 4.0cm 3.0cm
\scalebox{0.45}{\includegraphics[trim = 1.8cm 6.9cm 2.2cm 7.1cm, clip = true]{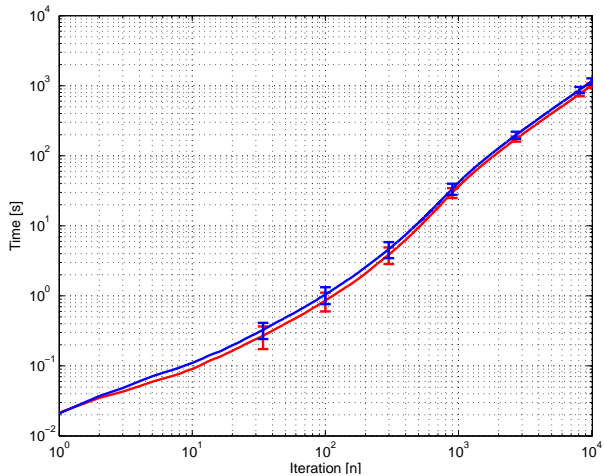}}
\caption{The time required for non-planning procedures to complete a certain number of iterations for the second problem set. The time curve for \AlgRRTsharp{} and PI-\AlgRRTsharp{} are shown in blue and red, respectively. Vertical bars denote standard deviation averaged over 100 trials.} \label{figure:non_planning_time_pt3}
%\caption{The change in the cost of the best solutions computed by the \AlgRRTstar{}, \AlgRRTsharp{} and variant algorithms. Vertical bars and filled circles denote standard deviation and completion time averaged over 100 trials, respectively.} \label{figure:time_cost_mean_d6_all}
\end{figure}
}

{
\setlength{\intextsep}{4pt}
\begin{figure}[htp]
\centering
% l b r t %4.0cm 3.0cm 4.0cm 3.0cm
\scalebox{0.45}{\includegraphics[trim = 1.8cm 6.9cm 2.2cm 7.1cm, clip = true]{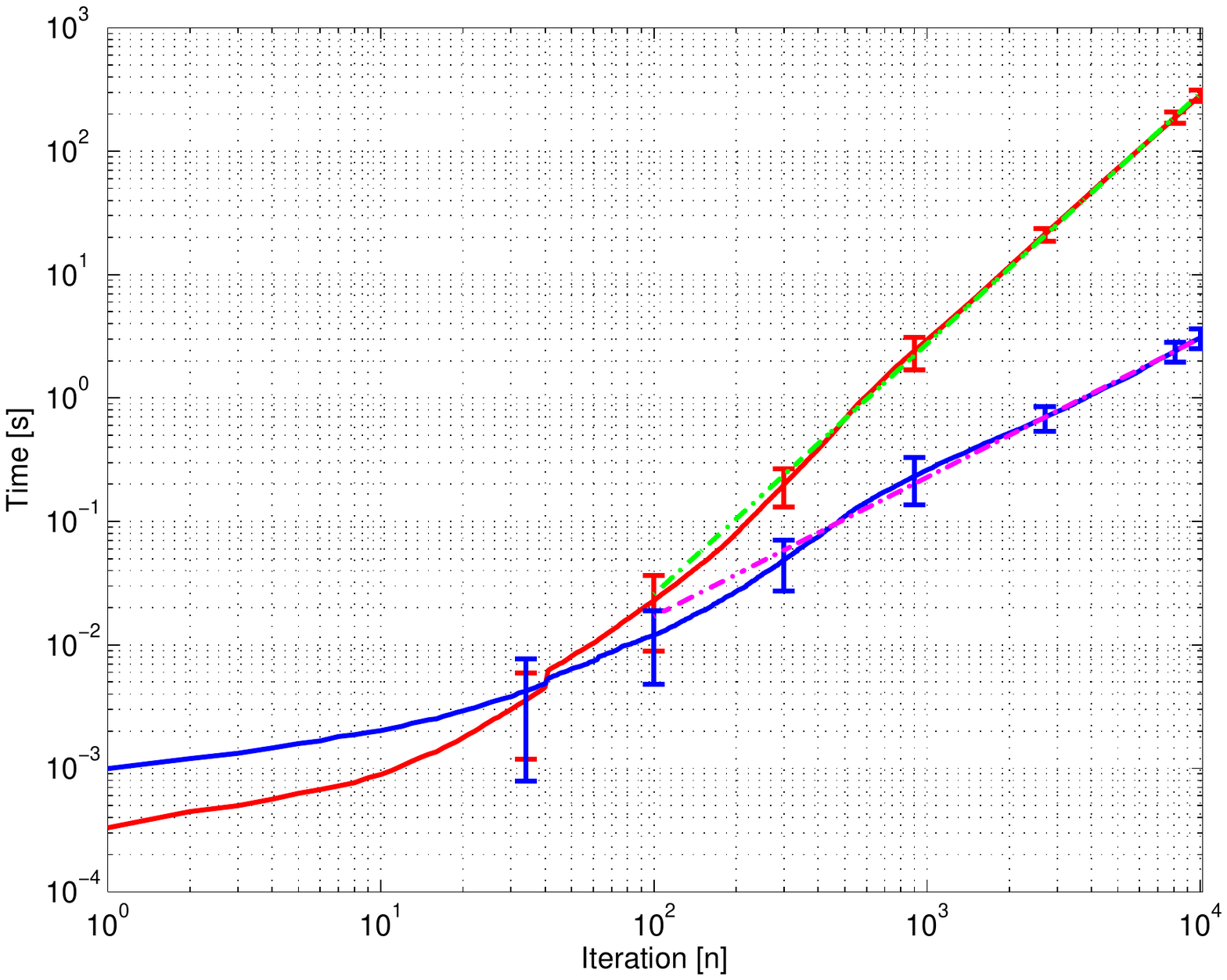}}
\caption{The time required for planning procedures to complete a certain number of iterations for the second problem set. The time curve for \AlgRRTsharp{} and PI-\AlgRRTsharp{} are shown in blue and red, respectively. Vertical bars denote standard deviation averaged over 100 trials.} \label{figure:planning_time_pt3}
%\caption{The change in the cost of the best solutions computed by the \AlgRRTstar{}, \AlgRRTsharp{} and variant algorithms. Vertical bars and filled circles denote standard deviation and completion time averaged over 100 trials, respectively.} \label{figure:time_cost_mean_d6_all}
\end{figure}
}

{
\setlength{\intextsep}{4pt}
\begin{figure}[htp]
\centering
% l b r t %4.0cm 3.0cm 4.0cm 3.0cm
\scalebox{0.45}{\includegraphics[trim = 1.8cm 6.9cm 2.2cm 7.1cm, clip = true]{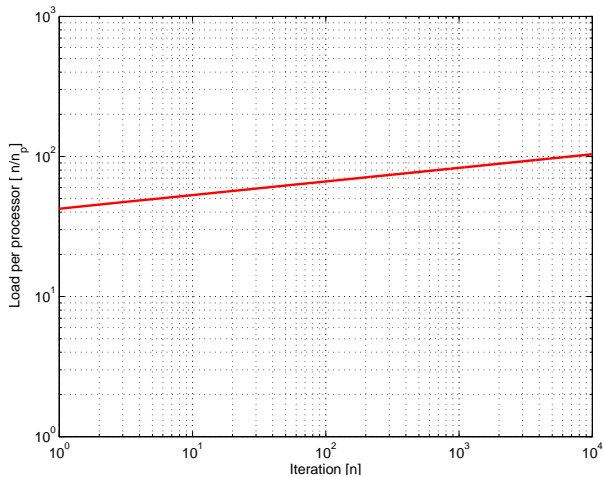}}
\caption{The time required for non-planning procedures to complete a certain number of iterations for the second problem set. The time curve for \AlgRRTsharp{} and PI-\AlgRRTsharp{} are shown in blue and red, respectively. Vertical bars denote standard deviation averaged over 100 trials.} \label{figure:load_curve_pt3}
%\caption{The change in the cost of the best solutions computed by the \AlgRRTstar{}, \AlgRRTsharp{} and variant algorithms. Vertical bars and filled circles denote standard deviation and completion time averaged over 100 trials, respectively.} \label{figure:time_cost_mean_d6_all}
\end{figure}
}

\FloatBarrier
\newpage

\section{Conclusion}

We show that a connection between DP and RRGs may yield different types of sampling-based motion planning algorithms  that utilize ideas from dynamic programming.
These algorithms ensure asymptotic optimality (with probability one) as the number of samples tends to infinity.
Use of policy iteration, instead of value iteration during the exploitation step, may offer several advantages, such as completely parallel implementation, avoidance of sorting and maintaining a queue of all sampled vertices in the graph, etc. 
We have implemented these ideas in the replanning step of the \AlgRRTsharp{} algorithm.  
The proposed PI-\AlgRRTsharp{} algorithm can be massively parallelized,  which can be exploited by 
taking advantage of the recent computational and technological advances of GPUs. This is part of ongoing work.

\newpage

\bibliography{rrtsharp_pi_report_final}

%\bibliography{references/iros2013_ref0,references/ref0,references/ref1,references/cowlagi}
%\bibliographystyle{ieeetr}
\bibliographystyle{plain}

\end{document}

%% file: packages.tex
\newenvironment{proof}[1][Proof]{\begin{trivlist}
\item[\hskip \labelsep {\bfseries #1}]}{\end{trivlist}}

\usepackage{theorem}
\usepackage{enumerate}

\newtheorem{theorem}{Theorem}
\newtheorem{lemma}{Lemma}

\usepackage{tikz}
\usetikzlibrary{fit,calc}
\usepackage{relsize}
\newcommand{\tikzmark}[1]{\tikz[overlay,remember picture,baseline] \node [anchor=base] (#1) {};}

\usepackage[colorlinks]{hyperref}
\hypersetup{
   bookmarksnumbered,
   pdfstartview={FitH},
   citecolor={black},
   linkcolor={black},
   urlcolor={black},
   pdfpagemode={UseOutlines}
}

\usepackage[colorlinks]{hyperref}

\hypersetup{
   bookmarksnumbered,
   pdfstartview={FitH},
   citecolor={black},
   linkcolor={black},
   urlcolor={black},
   pdfpagemode={UseOutlines}
}

\usepackage{dsfont}
\usepackage{placeins}
\usepackage{color}
\usepackage{mathtools}
\usepackage{url}
\usepackage{exscale}
\usepackage{subfigure}
\usepackage{caption}
\usepackage{array}
\usepackage{epsfig}
\usepackage{xspace}
\usepackage{bm}
\usepackage{amsmath}

\usepackage{amssymb}
\usepackage{stmaryrd}
\usepackage{comment}
\usepackage{fancybox}
\usepackage{multirow}
\usepackage[vlined, linesnumbered, boxruled]{algorithm2e}
\usepackage{relsize}
\usepackage{rotating}

\usepackage[margin=0.75in,top=0.75in,nohead]{geometry}

%% file: commands.tex
\definecolor{gray}{RGB}{190,190,190}
\definecolor{darkgreen}{rgb}{0,0.5,0} 
\definecolor{fullred}{rgb}{0.85,.0,.1} 
\definecolor{brown}{rgb}{0.65,0.16,0.16}

\newcommand{\AlgRRT}{\ensuremath{{\mathrm{RRT}}}}
\newcommand{\AlgPRM}{\ensuremath{{\mathrm{PRM}}}}
\newcommand{\AlgRRG}{\ensuremath{{\mathrm{RRG}}}}
\newcommand{\AlgPRMstar}{\ensuremath{{\mathrm{PRM}^*}}}
\newcommand{\AlgRRTstar}{\ensuremath{{\mathrm{RRT}^*}}}

\newcommand{\AlgRRTsharp}{\ensuremath{{\mathrm{RRT}^{\tiny \#}}}}

\newcommand{\AlgAstar}{\ensuremath{{\mathrm{A}^*}}}

\DeclareMathOperator*{\argmin}{arg\,min}

\newcommand{\reals}{\mathbb{R}}  % Reals
\newcommand{\naturals}{\mathbb{N}} 	% Natural numbers
\newcommand{\X}{\mathcal{X}} 		%Configuration space
\newcommand{\Xfree}{\X_{\mathrm{free}}} %Free Configuration space
\newcommand{\Xgoal}{\X_{\mathrm{goal}}} % Goal set in the configuration space
\newcommand{\Xobs}{\X_{\mathrm{obs}}} 	% Obstacle region in the configuration space
\newcommand{\succx}{x^{\prime}}	% randomly sampled state x
\newcommand{\xinit}{x_{\mathrm{init}}}	% initial state x	
\newcommand{\xgoal}{x_{\mathrm{goal}}}	% goal state x	
\newcommand{\graph}[1]{\mathcal{#1}}		% Graph data structure

\newcommand{\Vgoal}{V_{\mathrm{goal}}}	% Goal vertex set

\newcommand{\Vprom}{V_{\mathrm{prom}}}	% Promising vertices
\newcommand{\Tmuk}{T_{\mu^{k}}}
\newcommand{\Jmuk}{J_{\mu^{k}}}
\newcommand{\Jmuki}{J_{\mu^{k,i}}}

\newcommand{\duration}{\mathtt{len}} 
\newcommand{\costValue}{\mathtt{c}}			% cost-value of the path

\newcommand{\PrcExtend}{\mathtt{Extend}} % Extend procedure
\newcommand{\PrcReplan}{\mathtt{Replan}} %ReduceInconsistency procedure
\newcommand{\PrcSuccessor}{\mathtt{succ}} % Successor procedure for a vertex x on a graph
\newcommand{\PrcPredecessor}{\mathtt{pred}} % Predecessor procedure for a vertex x on a graph
\newcommand{\hideMaterial}[1]{}
\newcommand{\hideoldproofofnegresult}[1]{ }

%% file: rrtsharp_omp_pi.tex
\IncMargin{1mm}

\begin{algorithm}[h]

    % \tiny
	% \scriptsize
	% \footnotesize
 	
 	%\tiny
 	\small
    \DontPrintSemicolon
    \SetKwInOut{Input}{input}
    \SetKwInOut{Output}{output}
    \SetKwBlock{NoBegin}{}{end}

    % Begin : Declaration of Functions for RRTsharp

    %\SetFuncSty{ptm}    % Adobe Times
    \SetKwFunction{fRRTsharp}{\ensuremath{\mathbf{PI}}-\ensuremath{\mathbf{RRT^{\#}}}}
    \SetKwFunction{fSample}{Sample}
    \SetKwFunction{fExtend}{Extend}
    \SetKwFunction{fReduceInconsistency}{Replan}
    \SetKwFunction{fParent}{parent}
    %\SetKwFunction{}{}

    % End : Declaration of Functions for RRTsharp

    % Begin : Declaration of Data Section for RRTsharp
    %\SetKwData{vCVProm}{$V_{\mathrm{p}}$}
    \SetKwData{vCVProm}{$B$}
    \SetKwData{vV}{$V$}
    \SetKwData{vE}{$E$}
    \SetKwData{vEPrime}{$E^{\prime}$}
    \SetKwData{vXInit}{$x_{\mathrm{init}}$}
    \SetKwData{vI}{$k$}
    \SetKwData{vXRand}{$x_{\mathrm{rand}}$}
    \SetKwData{vT}{$\mathcal{T}$}
    \SetKwData{vG}{$\mathcal{G}$}
    \SetKwData{vCXGoal}{$x_{\mathrm{goal}}$}
    %\SetKwData{vCXGoal}{$\mathcal{X}_{\mathrm{goal}}$}
    \SetKwData{vCX}{$\mathcal{X}$}
    \SetKwData{vX}{$x$}

%   \SetKwData{}{}

    % End : Declaration of Data Section for RRTsharp

    \SetFuncSty{textbf}
    %\vT $\leftarrow$ \fRRTsharp{\vXInit, \vCXGoal, \vCX}
    
    \fRRTsharp{\vXInit, \vCXGoal, \vCX}
    \SetFuncSty{texttt}
    \NoBegin
    {
        $\vV \leftarrow \{\vCXGoal\}$;
        %$\vV \leftarrow \{\vXInit\}$;
        $\vCVProm \leftarrow \vV$;
        $\vE \leftarrow \emptyset$;

        %$\vT \leftarrow (\vV,\vE)$;
        $\vG \leftarrow (\vV,\vE)$;
        %$\vI \leftarrow 0$;

        \For{$\vI = 1$ to $N$ \label{line:rrtsharp_itbegin}}
        {
            $\vXRand = \fSample(\vI)$;

            $(\vG, \vCVProm^{\prime}) \leftarrow \fExtend(\vG, \vCVProm, \vXInit, \vXRand)$;

			\tikzmark{a0}\If{$|\vCVProm^{\prime}| > |\vCVProm|$}
			{
				$\vCVProm \leftarrow \fReduceInconsistency(\vG,\vCVProm^{\prime}, \vXInit, \vCXGoal)$;\tikzmark{a1}\label{line:rrtsharp_itend}
			}

        }

        $(\vV,\vE) \leftarrow \vG$;
        $\vEPrime \leftarrow \emptyset$;

        \ForEach{$\vX \in \vV$}
        {
            $\vEPrime \leftarrow \vEPrime \cup \{(\vX, \fParent(\vX))\}$;
        }

        \Return{$\vT = (\vV,\vEPrime)$};
    }

\caption{{\small Body of the \AlgRRTsharp{} Algorithm (PI)}} \label{alg:rrtsharp_omp_pi}

\end{algorithm}
%\DecMargin{1em}

%% file: extend_rrtsharp_omp_pi.tex
\IncMargin{1mm}

\begin{algorithm}[h]
    % \tiny
    %\scriptsize
	\footnotesize
	%\small
    %\scriptsize
    %\SetAlFnt{\small\sf}

    \DontPrintSemicolon
    \SetKwInOut{Input}{input}
    \SetKwInOut{Output}{output}
    \SetKwBlock{NoBegin}{}{end}

    % Begin : Declaration of Functions for Extend

    %\SetFuncSty{ptm}    % Adobe Times
    %\SetKwFunction{fExtend}{${\tt Extend_\AlgRRTsharp}$}  
    \SetKwFunction{fUpdateList}{UpdateList}
	\SetKwFunction{fExtend}{Extend}      
    \SetKwFunction{fNearest}{Nearest}
    \SetKwFunction{fSteer}{Steer}
    \SetKwFunction{fObstacleFree}{ObstacleFree}
    \SetKwFunction{fUpdateState}{UpdateQueue}
    \SetKwFunction{fNear}{Near}
    \SetKwFunction{fG}{g}
    \SetKwFunction{fC}{c}
    \SetKwFunction{fH}{h}
    \SetKwFunction{fLMC}{lmc}
    \SetKwFunction{fParent}{parent}
    \SetKwFunction{fInitState}{Initialize}
    \SetKwFunction{fCJ}{J}
    \SetKwFunction{fPred}{pred}
    %\SetKwFunction{}{}

    % End : Declaration of Functions for Extend

    % Begin : Declaration of Data Section for Extend
    
    \SetKwData{vCVProm}{$V_{\mathrm{p}}$}
    \SetKwData{vCVProm}{$B$}
    \SetKwData{vCJ}{$J$}
    \SetKwData{vCJmin}{$J_{\min}$}
    
    \SetKwData{vXInit}{$x_{\mathrm{init}}$}    
        
    \SetKwData{vT}{$\mathcal{T}$}
    \SetKwData{vG}{$\mathcal{G}$}
    \SetKwData{vTPrime}{$\mathcal{T}^{\prime}$}
    \SetKwData{vGPrime}{$\mathcal{G}^{\prime}$}
    \SetKwData{vX}{$x$}
    \SetKwData{vXRand}{$x_{\mathrm{rand}}$}
    \SetKwData{vV}{$V$}
    \SetKwData{vE}{$E$}
    
    %\SetKwData{vVPrime}{$V^{\prime}$}
    \SetKwData{vEPrime}{$E^{\prime}$}
    \SetKwData{vXNearest}{$x_{\mathrm{nearest}}$}
    \SetKwData{vXNew}{$x_{\mathrm{new}}$}
    \SetKwData{vCXNear}{$\mathcal{X}_{\mathrm{near}}$}
    \SetKwData{vXNear}{$x_{\mathrm{near}}$}

%    \SetKwData{}{}

    % End : Declaration of Data Section for Extend
	%\Indentp{-0.5em}
    \SetFuncSty{textbf}
    %\vTPrime $\leftarrow$ \fExtend{\vT,\vX}
    \fExtend{\vG, \vCVProm, \vXInit, \vXRand}
    \SetFuncSty{texttt}
    \NoBegin
    {
%        $\vVPrime \leftarrow \vV$;
%        $\vEPrime \leftarrow \vE$;
		
		\tikzmark{p0st}$(\vV,\vE) \leftarrow \vG$;
		$\vEPrime \leftarrow \emptyset$;
        
        $\vXNearest = \fNearest(\vG,\vXRand)$;
         
        $\vXNew = \fSteer(\vXRand, \vXNearest)$;\tikzmark{p0en}
        
        \If{$\fObstacleFree(\vXNew, \vXNearest)$}
        {
            %$\fG(\vXNew) \leftarrow \infty$;
            %$\fLMC(\vXNew) \leftarrow \infty$;
            %\fInitState{\vXNew, \vXNearest};
            
            $\fCJ(\vXNew) = \fC(\vXNew, \vXNearest) + \fCJ(\vXNearest)$; 
            
            $\fParent(\vXNew) = \vXNearest$; 
            
            %\tikzmark{p1st}$\fG(\vXNew) \leftarrow \infty$;
			
			%$\fLMC(\vXNew) = \fG(\vXNearest) + \fC(\vXNearest, \vXNew)$;
						
			%$\fParent(\vXNew) = \vXNearest$;\tikzmark{p1en}
			          
          	\tikzmark{p2st}$\vCXNear \leftarrow \fNear(\vG,\vXNew,|\vV|)$;\tikzmark{p2en}
          	
          	%$\vCJmin(\vXNew) = \vCJ(\vXNew)$;
			
            \tikzmark{p3st}\ForEach{$\vXNear \in \vCXNear$}
            {
                \tikzmark{p4st}\If{$\fObstacleFree(\vXNew, \vXNear)$}
                {
                	
                    \If{ $\fCJ(\vXNew) > \fC(\vXNew, \vXNear) + \fCJ(\vXNear)$}
                    {
						$\fCJ(\vXNew) = \fC(\vXNew, \vXNear) + \fCJ(\vXNear)$;                        

                        $\fParent(\vXNew) = \vXNear$;
                    }

                    $\vEPrime \leftarrow \vEPrime \cup \{ (\vXNew, \vXNear), (\vXNear, \vXNew) \}$;\tikzmark{p3en}
                }
            }
			
			\tikzmark{p5st}\If{$\fH(\vXInit, \fParent(\vXNew))+ \fCJ(\fParent(\vXNew)) < \fCJ(\vXInit) $}
	        {
				$\vCVProm \leftarrow \vCVProm \cup \{\vXNew\}$;
	        }
			            
			%$\vV \leftarrow \vV \cup \{ \vXNew \} \cup \fPred(\vG, \vXNew)$;

			$\vV \leftarrow \vV \cup \{ \vXNew \}$;

            $\vE \leftarrow \vE \cup \vEPrime$;\tikzmark{p5en}

            %$\vE \leftarrow \vE \cup \{ (\fParent(\vXNew), \vXNew) \}$;

            %$\fUpdateList(\vXNew)$;\tikzmark{p5en}

        }
	
		%$\vGPrime \leftarrow (\vV,\vE)$;
	
        %\Return{$(\vGPrime, \vCVProm)$};
    	$\vG \leftarrow (\vV,\vE)$;
      	
        \Return{$(\vG, \vCVProm)$};
    }

%    \SetFuncSty{textbf}
%    %\vTPrime $\leftarrow$ \fExtend{\vT,\vX}
%    \fInitState{\vX}
%    \SetFuncSty{texttt}
%    \NoBegin
%    {
%        $\fG(\vX) \leftarrow \infty$;
%
%        $\fLMC(\vX) \leftarrow \infty$;
%
%        $\fParent(\vX) \leftarrow \varnothing$;
%    }
%\caption{${\tt Extend}$ Procedure} \label{alg:extend_rrtsharp}
\caption{{\small ${\tt Extend}$ Procedure for \AlgRRTsharp{} (PI)}}\label{alg:extend_rrtsharp_omp_pi} 
\end{algorithm}
%\DecMargin{1em} 

%% file: replan_omp_pi.tex
\IncMargin{1mm}
\begin{algorithm}[h]
    %\tiny
    \small
    \DontPrintSemicolon
    \SetKwInOut{Input}{input}
    \SetKwInOut{Output}{output}
    \SetKwBlock{NoBegin}{}{end}
	\SetKwFor{Loop}{Loop}{}{EndLoop}

    % Begin : Declaration of Functions for ReduceInconsistency

    \SetFuncSty{ptm}    % Adobe Times

    \SetKwFunction{fReplan}{Replan}
    \SetKwFunction{fKey}{Key}
    \SetKwFunction{fLMC}{lmc}

    \SetKwFunction{fG}{g}
    \SetKwFunction{fSucc}{succ}
    \SetKwFunction{fC}{c}
    \SetKwFunction{fCJ}{J}
    \SetKwFunction{fParent}{parent}
    \SetKwFunction{fUpdateState}{UpdateQueue}
    
    \SetKwFunction{fEvaluate}{Evaluate}
	\SetKwFunction{fImprove}{Improve}
	
    %\SetKwFunction{}{}

    % End : Declaration of Function for ReduceInconsistency

    % Begin : Declaration of Data Section for ReduceInconsistency
    %\SetKwData{vVPrime}{$V^{\prime}$}
    %\SetKwData{vEPrime}{$E^{\prime}$}
    
    \SetKw{Break}{break}
        
    \SetKwData{vCJ}{$J$}
    \SetKwData{vCJPrime}{$J^{\prime}$}
    \SetKwData{vCS}{$S$}
    
    %\SetKwData{vCVProm}{$V_{\mathrm{p}}$}    
    \SetKwData{vCVProm}{$B$} 
    \SetKwData{vV}{$V$}
    \SetKwData{vE}{$E$}
    \SetKwData{vFrontier}{$q$}
    \SetKwData{vXGoal}{$x_{\mathrm{goal}}$}
    \SetKwData{vT}{$\mathcal{T}$}
    \SetKwData{vG}{$\mathcal{G}$}
    \SetKwData{vTPrime}{$\mathcal{T}^{\prime}$}
    \SetKwData{vGPrime}{$\mathcal{G}^{\prime}$}
    %\SetKwData{vCXGoal}{$\mathcal{X}_{\mathrm{goal}}$}
    \SetKwData{vCXGoal}{$x_{\mathrm{goal}}$}
    \SetKwData{vXminGoal}{$v^{*}_{\mathrm{goal}}$}
    \SetKwData{vVminGoal}{$v^{*}_{\mathrm{goal}}$}
    \SetKwData{vX}{$x$}
    \SetKwData{vS}{$s$}
    \SetKwData{vV}{$v$}
	\SetKwData{vXInit}{$x_{\mathrm{init}}$}

%    \SetKwData{}{}
    % End : Declaration of Data Section for ReduceInconsistency

    \SetFuncSty{textbf}
    %\vT $\leftarrow$ \fReplan{\vT,\vCXGoal}
    \fReplan {\vG, \vCVProm, \vXInit, \vCXGoal} % ,\vCXGoal
    \SetFuncSty{texttt}
    \NoBegin
    {
        %$(\vV,\vE) \leftarrow \vG$;

        %\While{$\vFrontier.findmin() \prec \fKey(\vXGoal)$ or $\fLMC(\vXGoal) > \fG(\vXGoal)$}
        %\tikzmark{p6st}\While{$\vFrontier.findmin() \prec \fKey(\vXminGoal)$}

        \tikzmark{p6st} \Loop{}
        {
            %\tikzmark{p7st}$\vX = \vFrontier.findmin()$;

            %\If{$\fG(\vX) > \fLMC(\vX)$}
            %{
          	%$\fG(\vX) = \fLMC(\vX)$;

            %$\vFrontier.delete(\vX)$;\tikzmark{p7en}

			\ForEach{ $\vX \in \vCVProm$}
			{  
				$\vCJPrime =  \fCJ(\vX)$;
				
				\ForEach{$\vV \in \fSucc(\vG,\vX)$}
				{
					\If{ $\vCJPrime > \fC (\vX, \vV) + \fCJ(\vV)$}
					{
						$\vCJPrime = \fC (\vX, \vV)$ + $\fCJ(\vV)$;
					
						$\fParent(\vX) = \vV$;
					}			
				}
				
				$\Delta J (\vX)= \fCJ(\vX)-\vCJPrime$;
				%\If{$\Delta J < \fCJ(\vX)-\vCJPrime$}
				%{
				%	$\Delta J = \fCJ(\vX)-\vCJPrime$;	
				%}

			}
			
			%\If{$\Delta J \leq \epsilon$}
			\If{$\max_{\vX \in \vCVProm}{\Delta J(\vX)} \leq \epsilon$}
			{
				\Return{$\vCVProm$};
			}
			$\vCVProm \leftarrow \fEvaluate(\vG, \vXInit, \vCXGoal)$;
			
%            \tikzmark{p8st}\ForEach{$\vV \in \fSucc(\vG,\vX)$}
%            {
%              \If{$\fLMC(\vV) > \fG(\vX) + \fC(\vX,\vV)$}
%              {
%                    %$\vE \leftarrow \vE \setminus \{(\fParent(\vV),\vV)\}$;
%
%                    %$\vE \leftarrow \vE \cup \{(\vX,\vV)\}$;
%
%                  \tikzmark{p9st}$\fLMC(\vV) = \fG(\vX) + \fC(\vX,\vV)$; 
%
%                  $\fParent(\vV) = \vX$;
%
%                  $\fUpdateState(\vV)$;\tikzmark{p6en}
%               }
%             }
            %}

        }

        %\Return{$\vGPrime = (\vV,\vE)$}
    }

%    \SetFuncSty{textbf}
%    \fUpdateState{\vX}
%    \SetFuncSty{texttt}
%    \NoBegin
%    {
%         \If{$\fG(\vX) \neq \fLMC(\vX)$ and $\vX \in \vFrontier$}
%         {
%            $\vFrontier.update(\vX, \fKey(\vX))$;
%         }
%         \ElseIf{$\fG(\vX) \neq \fLMC(\vX)$ and $\vX \notin \vFrontier$}
%         {
%            $\vFrontier.insert(\vX, \fKey(\vX))$;
%         }
%         \ElseIf{$\fG(\vX) = \fLMC(\vX)$ and $\vX \in \vFrontier$}
%         {
%            $\vFrontier.delete(\vX)$;
%         }
%    }
%
%    \SetKwFunction{fH}{h}
%    %\SetKwFunction{}{}
%
%
%    \SetKwData{vGPrime}{$g^{\prime}$}
%    \SetKwData{vGMin}{$g_{\min}$}
%    \SetKwData{vF}{$f$}
%    \SetKwData{vKey}{$key$}
%
%    \SetFuncSty{textbf}
%    \fKey{\vV}
%    \SetFuncSty{texttt}
%    \NoBegin
%    {
%        $\vGMin = \min(\fG(\vV),\fLMC(\vV))$;
%
%        $\vF = \vGMin + \fH(\vV)$;
%
%        \Return{$\vKey = (\vF,\vGMin)$};
%    }

\caption{ {\small ${\tt Replan}$ Procedure (PI)} {\color{white} $^{\#}$}} \label{alg:replan_omp_pi}
\end{algorithm}
%\DecMargin{1em} 

%% file: omp_policy_evaluation.tex
\IncMargin{1mm}
\begin{algorithm}[h]
    %\tiny
    \small
    \DontPrintSemicolon
    \SetKwInOut{Input}{input}
    \SetKwInOut{Output}{output}
    \SetKwBlock{NoBegin}{}{end}

    % Begin : Declaration of Functions for ReduceInconsistency

    \SetFuncSty{ptm}    % Adobe Times
    \SetKwFunction{fReplan}{Replan}
    \SetKwFunction{fKey}{Key}
    \SetKwFunction{fLMC}{lmc}

 	\SetKwFunction{fChildren}{children}
    \SetKwFunction{fG}{g}
    \SetKwFunction{fSucc}{succ}
    \SetKwFunction{fC}{c}
    \SetKwFunction{fH}{h}
    \SetKwFunction{fCJ}{J}
    \SetKwFunction{fParent}{parent}
    \SetKwFunction{fUpdateState}{UpdateQueue}
    
    \SetKwFunction{fEvaluate}{Evaluate}
	\SetKwFunction{fImprove}{Improve}
	
    %\SetKwFunction{}{}

    % End : Declaration of Function for ReduceInconsistency

    % Begin : Declaration of Data Section for ReduceInconsistency
    %\SetKwData{vVPrime}{$V^{\prime}$}
    %\SetKwData{vEPrime}{$E^{\prime}$}

    %\SetKwData{vCVProm}{$V_{\mathrm{p}}$} 
    \SetKwData{vCVProm}{$B$} 
    \SetKwData{vCJ}{$J$}
    \SetKwData{vCJmin}{$J_{\min}$}
    \SetKwData{vXInit}{$x_{\mathrm{init}}$}
    \SetKwData{vListL}{$q$}
    \SetKwData{vListS}{$S$}
        
    \SetKwData{vV}{$V$}
    \SetKwData{vE}{$E$}
    \SetKwData{vFrontier}{$q$}
    \SetKwData{vXGoal}{$x_{\mathrm{goal}}$}
    \SetKwData{vT}{$\mathcal{T}$}
    \SetKwData{vG}{$\mathcal{G}$}
    \SetKwData{vTPrime}{$\mathcal{T}^{\prime}$}
    \SetKwData{vGPrime}{$\mathcal{G}^{\prime}$}
    %\SetKwData{vCXGoal}{$\mathcal{X}_{\mathrm{goal}}$}
    \SetKwData{vCXGoal}{$x_{\mathrm{goal}}$}
    \SetKwData{vXminGoal}{$v^{*}_{\mathrm{goal}}$}
    \SetKwData{vVminGoal}{$v^{*}_{\mathrm{goal}}$}
    \SetKwData{vX}{$x$}
    \SetKwData{vS}{$v$}

%    \SetKwData{}{}
    % End : Declaration of Data Section for ReduceInconsistency

    \SetFuncSty{textbf}
    %\vT $\leftarrow$ \fReplan{\vT,\vCXGoal}
    \fEvaluate{\vG, \vXInit, \vCXGoal}
    \SetFuncSty{texttt}
    \NoBegin
    {
        %$(\vV,\vE) \leftarrow \vG$;

        %\While{$\vFrontier.findmin() \prec \fKey(\vXGoal)$ or $\fLMC(\vXGoal) > \fG(\vXGoal)$}
        %\tikzmark{p6st}\While{$\vFrontier.findmin() \prec \fKey(\vXminGoal)$}
        $(\vV,\vE) \leftarrow \vG$;
        
        \If{$\vXInit \in V$}
        {
	        $\vX = \vXInit$;
	        
	        \While{ $\vX \neq \vXGoal$ }
	        {
	        	$\vListL.push\_front(x)$;
	        	
	        	$x = \fParent(x)$;
	        }
	        
	        %$\vX = \vXGoal$;
	        
	        $\fCJ(\vXGoal) = 0$;
	        
	        %$\vX = \vListLpop\_front()$;
	        
	        \While{$\vListL.empty()$}
	        {
	        	$\vX = \vListL.pop\_front()$;
	        	
	        	$\fCJ(\vX) = \fC(\vX, \fParent(\vX)) + \fCJ(\fParent(\vX))$;
	        }
        }
        \Else
        {
	        $\fCJ(\vXInit) = \infty$;
        }
        
        $\vCVProm \leftarrow \{\vXGoal\}$;
        
        $\vListL.push\_back(\vXGoal)$;
        
        %\While{$\neg \vListL.empty()$}
        \While{$\vListL.nonempty()$}
        {
        	$\vX = \vListL.pop\_front()$;
        	
        	\If{ $ \fH(\vXInit, \vX)+ \fCJ(\vX) < \fCJ(\vXInit)$ }
        	{
	        	\ForEach{$\vS \in \fPred(\vG,\vX)$}
				{
					\If{ $ \fParent(\vS) = \vX$ }
					{
						$\fCJ(\vS) = \fC (\vS, \vX)$ + $\fCJ(\vX)$;
					}

					$\vCVProm \leftarrow \vCVProm \cup \{\vS\}$;
					
					$\vListL.push\_back(\vS)$;
					
					%$\vCVProm.push\_back(\vS)$;
								
				}
        	}
        }

        \Return{$\vCVProm$};
    }

%    \SetFuncSty{textbf}
%    \fUpdateState{\vX}
%    \SetFuncSty{texttt}
%    \NoBegin
%    {
%         \If{$\fG(\vX) \neq \fLMC(\vX)$ and $\vX \in \vFrontier$}
%         {
%            $\vFrontier.update(\vX, \fKey(\vX))$;
%         }
%         \ElseIf{$\fG(\vX) \neq \fLMC(\vX)$ and $\vX \notin \vFrontier$}
%         {
%            $\vFrontier.insert(\vX, \fKey(\vX))$;
%         }
%         \ElseIf{$\fG(\vX) = \fLMC(\vX)$ and $\vX \in \vFrontier$}
%         {
%            $\vFrontier.delete(\vX)$;
%         }
%    }
%
%    \SetKwFunction{fH}{h}
%    %\SetKwFunction{}{}
%
%
%    \SetKwData{vGPrime}{$g^{\prime}$}
%    \SetKwData{vGMin}{$g_{\min}$}
%    \SetKwData{vF}{$f$}
%    \SetKwData{vKey}{$key$}
%
%    \SetFuncSty{textbf}
%    \fKey{\vS}
%    \SetFuncSty{texttt}
%    \NoBegin
%    {
%        $\vGMin = \min(\fG(\vS),\fLMC(\vS))$;
%
%        $\vF = \vGMin + \fH(\vS)$;
%
%        \Return{$\vKey = (\vF,\vGMin)$};
%    }

\caption{ {\small ${\tt Evaluate}$ Procedure (PI)} {\color{white} $^{\#}$}} \label{alg:omp_policy_evaluation}
\end{algorithm}
%\DecMargin{1em} 